\documentclass{article}

\usepackage{arxiv}

\usepackage[round]{natbib}

\usepackage{graphicx}
\usepackage{amsmath,amsfonts,amssymb,amsthm}
\usepackage{float}

\theoremstyle{plain}
\newtheorem{thm}{Theorem}
\newtheorem{lem}[thm]{Lemma}

\theoremstyle{definition}

\newtheorem{rem}[thm]{Remark}
\newtheorem{assump}[thm]{Assumption}

\newenvironment{acks}[1][Acknowledgments]{\section*{#1}}{}

\title{TAP Accuracy Below the Fluctuation Scale and Universal Posterior Geometry in Spherical Linear Models}

\author{
  Jingbo Liu\\
  Department of Statistics\\
  University of Illinois Urbana-Champaign\\
  Champaign, IL 61820\\
  \texttt{jingbol@illinois.edu}\\
  \And
  Zhiyuan Yu\\
  Department of Statistics\\
  University of Illinois Urbana-Champaign\\
  Champaign, IL 61820\\
  \texttt{yu124@illinois.edu}
}
\renewcommand{\shorttitle}{TAP accuracy and universal posterior geometry}

\begin{document}
\maketitle

\begin{abstract}
We study the Bayes-optimal spherical linear model as the ambient dimension $p$ and sample size grow proportionally, under a quantitative Marchenko--Pastur spectral-regularity condition on the design. This condition is satisfied by normalized i.i.d.\ designs with standardized entries of finite fourth moment, but does not require entrywise independence or impose conditions on the singular vectors. Under this condition, we prove a quantitative all-temperature TAP approximation and characterize the posterior geometry. For the natural finite-aspect-ratio TAP functional, the normalized spherical free energy and the TAP optimum differ by $O_P(p^{-1})$. Each is within $O_P(p^{-1/2})$ of its explicit deterministic equivalent, and this fluctuation scale is sharp. Uniformly over all global TAP maximizers, the normalized squared Euclidean distance to the spherical posterior mean is $O_P(p^{-1})$. We also prove that the posterior mass outside a data-dependent band determined by the ridge estimator has sharp exponential order. More precisely, uniformly over sufficiently small band widths $\varepsilon$, the logarithm of this mass is at most $-cp\varepsilon^2+O_P(1)$. For every fixed geometrically admissible width, a spherical-cap construction gives a matching exponential-order lower bound on this mass. For every deterministic sequence of widths $\varepsilon_p\gg p^{-1/2}$, the corresponding bands capture asymptotically all posterior mass.
\end{abstract}



\section{Introduction}
Random linear models form a basic framework for high-dimensional regression, signal recovery, and related inverse problems. In the Bayesian formulation, the posterior mean is the Bayes estimator under squared-error loss, while the posterior covariance quantifies the remaining uncertainty. The corresponding normalizing constant determines the free energy. Obtaining sharp asymptotic characterizations of these quantities when the sample size and ambient dimension grow proportionally is a central problem in statistics, information theory, and statistical physics \cite{banerjee2026bayesian,Barbier2020,fan2021tap,reeves2019replica}.

Even when the design matrix has i.i.d.\ entries, the posterior generally exhibits dependence across the signal coordinates. For Bayesian linear models with product priors, write $\Pi(\cdot\mid X,Y)$ for the posterior and let $\mathcal Q$ be a tractable variational class. A standard variational approach considers the optimization problem
\begin{align*}
\inf_{q\in\mathcal Q}\mathrm{KL}\bigl(q\|\Pi(\cdot\mid X,Y)\bigr),
\end{align*}
where $\mathrm{KL}$ denotes Kullback-Leibler divergence \cite{Blei_2017,wainwright2008graphical}. The naive mean-field approximation restricts $\mathcal Q$ to the fully factorized family
\begin{align*}
\mathcal Q_{\mathrm{MF}}:= \left\{q(d\beta)=\prod_{i=1}^p q_i(d\beta_i)\right\}.
\end{align*}
In proportional random designs, this approximation can yield inaccurate estimates of the log-normalizing constant and posterior mean. It can also underestimate marginal uncertainty \cite{celentano2023meanfieldvariationalinferencetap,qiu2023sub}. These failures motivate variational functionals that retain the leading effect of the dependence induced by the design matrix.

The Bethe and Thouless-Anderson-Palmer (TAP) free energies provide such corrections. The TAP approach originated in the Sherrington-Kirkpatrick model \cite{mezard1987spin,thouless1977solution}. Plefka derived the TAP equations from a second-order expansion of the Gibbs potential in the interaction strength \cite{plefka1982convergence}. This expansion concerns the construction of the TAP functional, rather than the finite-size approximation error studied here. Rigorous developments include iterative constructions of solutions to these equations under the de Almeida-Thouless stability condition \cite{bolthausen2014iterative}. TAP free energies also provide variational representations of spin-glass free energies \cite{belius2022high,chen2020generalizedtapfreeenergy,fan2021tap,subag2023tap}.

Variational free energies are closely connected to message-passing algorithms. Under the usual positivity and consistency conditions, stationary points of the Bethe free energy correspond to fixed points of belief propagation \cite{wainwright2008graphical,yedidia2001bethe,yedidia2003understanding}. In dense random models, TAP functionals play an analogous variational role, while approximate message passing provides an algorithmic counterpart \cite{fan2021tap,Krzakala2014}. In compressed sensing, Krzakala et al. made this connection explicit by formulating naive mean-field and Bethe free energies and relating the stationary equations of the Bethe functional to approximate message passing \cite{Krzakala2014}. More generally, state evolution describes the iterates of approximate message passing in the large-system limit \cite{bayati2011dynamics,Krzakala2014}.

Rigorous TAP results have also been established for specific high-dimensional inference models. Fan, Mei, and Montanari analyzed $\mathbb Z_2$ synchronization. In a sufficiently strong-signal regime, the rank-one matrices formed from low-energy TAP critical points approximate the posterior mean of the signal outer product, the sign-invariant estimand in that model \cite{fan2021tap}. For Bayesian linear models with compactly supported product priors, Celentano et al. proved, under a nondegeneracy condition, the existence of a local TAP minimizer that approximates the posterior marginals. They also established local convexity and algorithmic guarantees in suitable parameter regimes. In high-noise or sufficiently low-aspect-ratio regimes, they further proved global convexity \cite{celentano2023meanfieldvariationalinferencetap}.

For the Gaussian-design spherical linear model, Qiu and Sen proved a TAP representation of the free energy in a high-noise, or high-temperature, regime \cite{qiu2023tap}. Their proof analyzes the posterior on spherical caps but does not provide a quantitative finite-size rate. Together, these works suggest that an appropriately defined TAP functional can provide the correct variational description of proportional high-dimensional Bayesian models. However, they do not establish an all-temperature global variational formula for the spherical linear model.

At low temperatures, a local TAP optimizer or an approximate-message-passing fixed point is not necessarily globally optimal. Since the TAP functional can be nonconcave, even for Gaussian and spherical priors, the possibility of other extrema attaining larger values cannot be easily ruled out \cite{yu2025proof}. 
Our earlier conference paper resolved this global-optimality problem for the spherical model with a Gaussian design by constructing a concave ridge surrogate that dominates the TAP functional \cite{yu2025proof}. It bounded the distance of global and near-global TAP maximizers from the ridge estimator in terms of their optimality gap. The same paper also gave an explicit example of TAP nonconcavity for the spherical prior.
The present paper adopts substantially improved proofs than \cite{yu2025proof} that further establish the following results:
First, we show that the gap between the spherical free energy and the TAP optimum is smaller than the sample-fluctuations of these quantities by a factor of $p^{-1/2}$;
previously a similar sharp approximation of the TAP free energy for the Ising perceptron was conjectured by 
Ding and Sun \cite[(1.12)]{ding2025capacity}.
Second, we prove a sharp convergence of the global TAP maximizers to the spherical posterior mean.
Third, we prove a concentration of the  posterior on a spherical band, clarifying the geometric picture. 

\subsection{Problem setup and notation}
For vectors, $\|\cdot\|:=\|\cdot\|_2$ denotes the Euclidean norm and $\langle\cdot,\cdot\rangle$ denotes the Euclidean inner product. For matrices, $\|\cdot\|_{\mathrm{op}}$ and $\|\cdot\|_F$ denote the operator and Frobenius norms respectively. We use $\operatorname{Tr}(\cdot)$ for the trace and $I_k$ for the $k\times k$ identity matrix. For symmetric matrices, $A\preceq B$ means that $B-A$ is positive semidefinite. We write $x_+:=\max\{x,0\}$. The symbols $c$ and $C$ denote positive constants whose values may change from line to line.

For deterministic sequences $u_p$ and $r_p>0$, we write $u_p=O(r_p)$ if $|u_p|/r_p$ is bounded for all sufficiently large $p$, and $u_p=o(r_p)$ if $u_p/r_p\to0$. For positive sequences $u_p$ and $r_p$, we write $u_p\asymp r_p$ if both $u_p=O(r_p)$ and $r_p=O(u_p)$. All probabilistic order notation below is with respect to the joint law of the planted model specified next. For random variables $U_p$, we write $U_p=O_P(r_p)$ if $U_p/r_p$ is tight. Equivalently, for every $\eta>0$, there exist finite constants $M_\eta$ and $p_\eta$ such that
\begin{align*}
\mathbb P\bigl(|U_p|>M_\eta r_p\bigr)\le\eta \qquad\text{for all }p\ge p_\eta.
\end{align*}
We write $U_p=o_P(r_p)$ if $U_p/r_p\xrightarrow{\mathbb P}0$. The symbols $\xrightarrow{\mathbb P}$ and $\xrightarrow{d}$ denote convergence in probability and convergence in distribution respectively. When auxiliary random variables are introduced, the underlying probability law is enlarged to include their explicitly specified joint distribution.

We consider a Bayes-optimal linear model with a uniform spherical prior. This prior models a rotationally invariant signal with fixed energy, $\|\beta\|^2=p$. It also provides a natural non-product setting for studying TAP approximations: the fixed-norm constraint induces dependence across the coordinates, while rotational symmetry permits explicit geometric and random-matrix analysis. Consequently, the fully factorized KL formulation above does not apply directly to the resulting posterior. The appropriate variational object is therefore a spherical TAP functional.

For each $p$, let $n:=n_p$ and set $\alpha_p:=n/p$. We observe
\begin{align*}
Y=X\beta^\star+W,
\end{align*}
where $X=X_p\in\mathbb R^{n\times p}$, $\beta^\star$ is uniformly distributed on $S^{p-1}(\sqrt p)$, and $W\sim N(0,\Delta I_n)$ with $\Delta>0$. The matrix $X$, the signal $\beta^\star$, and the noise $W$ are mutually independent.

Throughout, we work in the proportional regime and assume that, for some fixed $\alpha\in(0,\infty)$,
\begin{align}
|\alpha_p-\alpha|=O\left(p^{-1/2}\right). \label{eq:aspect_ratio_rate}
\end{align}
This mild quantitative condition is satisfied, for example, when $n=\lfloor\alpha p\rfloor$, in which case $|\alpha_p-\alpha|=O(p^{-1})$.

Let $\mu_{\mathrm{MP},\alpha_p}$ denote the Marchenko--Pastur (MP) law on $[0,\infty)$ corresponding to the aspect ratio $p/n=1/\alpha_p$. We impose the following quantitative spectral condition on the design matrix.

\begin{assump}[Quantitative Marchenko--Pastur spectral regularity]
\label{assump:subgaussian_design}
There exists a deterministic constant $C_X<\infty$ such that
\begin{align*}
\mathbb P\left(\|X\|_{\mathrm{op}}\le C_X\right)\longrightarrow1.
\end{align*}
Moreover, for every fixed $\delta>0$,
\begin{align*}
\frac1p\log\det\left(I_p+\frac1\delta X^\top X\right) &=\int\log\left(1+\frac{x}{\delta}\right) d\mu_{\mathrm{MP},\alpha_p}(x)+O_P(p^{-1}),\\
\frac{\delta}{p}\operatorname{Tr}(X^\top X+\delta I_p)^{-1} &=\int\frac{\delta}{x+\delta} d\mu_{\mathrm{MP},\alpha_p}(x)+O_P(p^{-1/2}).
\end{align*}
\end{assump}
Assumption~\ref{assump:subgaussian_design} is verified for normalized i.i.d.\ designs with finite fourth moment in Lemma~\ref{lem_sectionA_2}. Its proof in Appendix~\ref{app:tap_estimates} uses the Bai--Yin spectral-edge theorem and Najim and Yao's linear-spectral-statistic estimates.

Apart from the independence condition stated above, Assumption~\ref{assump:subgaussian_design} concerns only the singular values of $X$ and allows both random and deterministic design sequences. This assumption is stronger than weak convergence of the empirical spectral distribution. Even when combined with a bounded spectral edge, weak convergence gives only an $o_P(1)$ normalized error for the two displayed statistics, without a quantitative rate. This is insufficient for the conclusions below the fluctuation scale. The quantitative estimates displayed above are needed in the proofs.

Let $\pi$ denote the uniform measure on $S^{p-1}(\sqrt p)$ and set
\begin{align}
Z_p^S &:=\int_{S^{p-1}(\sqrt p)} \exp\left\{-\frac{1}{2\Delta}\|Y-X\beta\|^2\right\}d\pi(\beta). \label{eq:spherical_partition_definition}
\end{align}
The spherical posterior is
\begin{align}
\Pi_S(d\beta\mid X,Y) &:=\frac{1}{Z_p^S} \exp\left\{-\frac{1}{2\Delta}\|Y-X\beta\|^2\right\}d\pi(\beta). \label{eq:spherical_posterior_definition}
\end{align}
The normalized spherical free energy is
\begin{align}
F_p^S:=\frac1p\log Z_p^S. \label{eq:spherical_free_energy_definition}
\end{align}
We study this free energy, the posterior mean, and the geometry of the posterior measure.

For $\|a\|<\sqrt p$, define the spherical TAP functional
\begin{align}
f_{\mathrm{TAP}}(a):={}& -\frac{1}{2\Delta p}\|Y-Xa\|^2 -\frac{\alpha_p}{2}\log\left(1+\frac{1-\|a\|^2/p}{\alpha_p\Delta}\right) +\frac12\log\left(1-\frac{\|a\|^2}{p}\right), \label{eq:tap_functional_definition}
\end{align}
and set $f_{\mathrm{TAP}}(a):=-\infty$ for $\|a\|\ge\sqrt p$. We further define
\begin{align}
F_{\mathrm{TAP}}:=\max_{\|a\|\le\sqrt p}f_{\mathrm{TAP}}(a), \qquad \mathcal M_p:=\operatorname*{arg\,max}_{\|a\|\le\sqrt p}f_{\mathrm{TAP}}(a). \label{eq:tap_optimum_definition}
\end{align}
For every fixed realization of $(X,Y)$, all terms other than $\frac12\log(1-\|a\|^2/p)$ remain bounded as $\|a\|\uparrow\sqrt p$, whereas this logarithmic term tends to $-\infty$. Hence, $f_{\mathrm{TAP}}(a)\to-\infty$ as $\|a\|\uparrow\sqrt p$. It follows that $\mathcal M_p$ is nonempty and compact, and the measurable maximum theorem ensures that the suprema over $\mathcal M_p$ used below are measurable.

\subsection{Contributions and main results}
For the Gaussian-design spherical model, Qiu and Sen proved $F_p^S-F_{\mathrm{TAP}}\xrightarrow{\mathbb P}0$ in a high-noise regime and left its extension to every fixed $\Delta>0$ as an open problem \cite{qiu2023tap}. Our earlier conference paper extended this convergence to every fixed $\Delta>0$ and localized global and near-global TAP maximizers around the ridge estimator \cite{yu2025proof}. The present work strengthens these results in four directions.
\begin{itemize}
\item \emph{TAP accuracy below the fluctuation scale.} For the finite-aspect-ratio TAP functional, we prove $|F_p^S-F_{\mathrm{TAP}}|=O_P(p^{-1})$, whereas the deviations of $F_p^S$ and $F_{\mathrm{TAP}}$ from their common deterministic equivalent occur on the sharp scale $p^{-1/2}$. Equivalently, the unnormalized TAP approximation error is $O_P(1)$, compared with individual sample fluctuations on the scale $\sqrt p$. This gives a samplewise comparison for the global spherical TAP optimum below the scale of its individual fluctuations. Ding and Sun conjectured an $O(1)$-accurate samplewise TAP description for the Ising perceptron \cite[(1.12)]{ding2025capacity}, while prior rigorous TAP representations generally identify the leading extensive term, with error $o_P(p)$ rather than $O_P(1)$ \cite{belius2022high,chen2020generalizedtapfreeenergy,qiu2023tap,yu2025proof}. Here, we establish an analogous $O(1)$-accurate TAP description for the spherical regression problem,
leveraging
a common sample-dependent ridge surrogate that captures the leading fluctuations.

\item \emph{Uniform identification of global TAP maximizers with the posterior mean.} We prove that the normalized squared Euclidean distance between the exact spherical posterior mean and any global TAP maximizer is $O_P(p^{-1})$, uniformly over all such maximizers. This strengthens the earlier ridge localization result by quantitatively identifying the global TAP maximizers with the Bayesian posterior mean itself.

\item \emph{Sharp exponential-order posterior geometry.} Uniformly over sufficiently small band widths $\varepsilon$, the log posterior mass outside a data-dependent ridge band is at most $-cp\varepsilon^2+O_P(1)$. For every fixed geometrically admissible width, a spherical-cap construction gives a matching exponential-order lower bound on this posterior mass. Thus the order-$p$ exponential speed is sharp for every fixed admissible width, while every sequence $\varepsilon_p\gg p^{-1/2}$ yields posterior contraction.

\item \emph{Spectral universality beyond i.i.d.\ designs.} The three results above hold for every random or deterministic design sequence satisfying Assumption~\ref{assump:subgaussian_design}. In particular, the proofs require neither entrywise independence nor conditions on the singular vectors. This extension does not rely on a new invariance principle. Instead, the reduction to the common Gaussian-prior ridge surrogate uses a high-probability spectral-edge bound and two quantitative Marchenko--Pastur linear-spectral-statistic estimates. Normalized i.i.d.\ designs with standardized entries of finite fourth moment provide a concrete example, as follows from established sample-covariance results \cite{Bai2010,najim2016gaussian}. The key point is that these spectral inputs suffice for the quantitative TAP and posterior-geometric conclusions obtained here. Weak convergence of the empirical spectrum alone is not sufficient for the stated rates. Related universality results for message-passing algorithms and regularized regression are obtained by different methods \cite{bayati2015universality,han2023universality}.
\end{itemize}

The free-energy proof combines radial conditioning and ridge domination. For the same observed data, exact radial conditioning and local density estimates show that the spherical and Gaussian-prior free energies differ by $O_P(p^{-1})$. Building on our earlier conference paper \cite{yu2025proof}, we bound the TAP functional by a strictly concave quadratic ridge functional. The resulting nonnegative gap vanishes quadratically at a deterministic reference value of the normalized squared norm. Consequently, $O_P(p^{-1/2})$ deviations of this quantity for the ridge estimator produce an $O_P(p^{-1})$ error. Gaussian integration and the log-determinant estimate in Assumption~\ref{assump:subgaussian_design} show that the Gaussian free energy also differs from the same ridge optimum by $O_P(p^{-1})$. The matching fluctuation lower bounds follow from an anti-concentration estimate for the Gaussian quadratic polynomial appearing in the ridge optimum and the $O_P(p^{-1})$ comparisons with $R_p$.

The posterior-mean approximation and ridge-band results do not rely on the TAP free-energy convergence theorem. For the posterior-mean result, Gaussian conditioning first shows that the spherical and Gaussian posterior means are within $O_P(1)$ of each other. The entropy-gap comparison and the quadratic curvature of the ridge surrogate then place every global TAP maximizer within $O_P(1)$ of the Gaussian posterior mean. For the ridge-band result, the conditioned-Gaussian representation and a linear exponential tilt yield the uniform upper bound, while a spherical-cap construction gives the matching exponential-order lower bound.

Here and throughout, ``all-temperature'' means that the results hold for every fixed $\Delta>0$. Constants may depend on $\Delta$, and no uniformity as $\Delta\downarrow0$ is claimed. Retaining $\alpha_p$ in \eqref{eq:tap_functional_definition} is essential for the below-fluctuation-scale comparison in Theorem~\ref{thm_sectionA_1}. Replacing $\alpha_p$ by $\alpha$ changes the functional by at most $C|\alpha_p-\alpha|$, uniformly over the domain on which it is finite. This bound is only of order $p^{-1/2}$ under \eqref{eq:aspect_ratio_rate}.

To state the deterministic equivalent, for $\gamma>0$ define
\begin{align*}
E(\gamma,\Delta) &:=\frac{1-\gamma-\gamma\Delta+ \sqrt{(\gamma\Delta+\gamma-1)^2+4\gamma\Delta}}{2},\\*
\phi(\gamma,\Delta) &:=-\frac{\gamma}{2}+\frac12\bigl(1-E(\gamma,\Delta)\bigr) -\frac{\gamma}{2}\log\left(1+\frac{E(\gamma,\Delta)}{\gamma\Delta}\right) +\frac12\log E(\gamma,\Delta).
\end{align*}
These quantities can be shown to equal the fixed point of the replica symmetric free energy and the corresponding free energy \cite[(45)]{yu2025proof}.
We abbreviate
\begin{equation}
E_p:=E(\alpha_p,\Delta),\qquad E_\Delta:=E(\alpha,\Delta). \label{eq:Ep_definition}
\end{equation}

The first main result quantifies the spherical TAP free-energy formula and shows that the $p^{-1/2}$ fluctuation scale around the finite-aspect-ratio deterministic equivalent is sharp.

\begin{thm}[TAP accuracy below the fluctuation scale]
\label{thm_sectionA_1}
Suppose that Assumption~\ref{assump:subgaussian_design} holds, $\alpha_p=n/p$ satisfies \eqref{eq:aspect_ratio_rate}, and $\Delta>0$ is fixed. Then
\begin{align}
\left|F_p^S-F_{\mathrm{TAP}}\right|&=O_P(p^{-1}), \label{eq:main_tap_accuracy}\\
F_p^S&=\phi(\alpha_p,\Delta)+O_P(p^{-1/2}), \label{eq:main_spherical_centering}\\
F_{\mathrm{TAP}}&=\phi(\alpha_p,\Delta)+O_P(p^{-1/2}). \label{eq:main_tap_centering}
\end{align}
Moreover, the $p^{-1/2}$ rate is sharp: there exist constants $c,\eta>0$, depending only on $(\alpha,\Delta,C_X)$, such that
\begin{align}
\liminf_{p\to\infty} \mathbb P\left( \left|F_p^S-\phi(\alpha_p,\Delta)\right| \ge \frac{c}{\sqrt p} \right)&\ge\eta, \label{eq:main_spherical_sharpness}\\
\liminf_{p\to\infty} \mathbb P\left( \left|F_{\mathrm{TAP}}-\phi(\alpha_p,\Delta)\right| \ge \frac{c}{\sqrt p} \right)&\ge\eta. \label{eq:main_tap_sharpness}
\end{align}
Consequently, neither centered quantity is $o_P(p^{-1/2})$. The bounds \eqref{eq:main_spherical_centering} and \eqref{eq:main_tap_centering} remain valid with $\phi(\alpha_p,\Delta)$ replaced by $\phi(\alpha,\Delta)$. No matching lower bound is asserted for the direct comparison $|F_p^S-F_{\mathrm{TAP}}|$.
\end{thm}

The second main result connects the TAP variational problem with the posterior mean.

\begin{thm}[Uniform approximation of the posterior mean by global TAP maximizers]
\label{thm_sectionA_2}
Under the assumptions of Theorem~\ref{thm_sectionA_1}, let
\begin{align*}
m_S:=\int\beta\,\Pi_S(d\beta\mid X,Y)
\end{align*}
be the spherical posterior mean. Then
\begin{align*}
\sup_{a^*\in\mathcal M_p}\frac1p\|a^*-m_S\|_2^2 =O_P(p^{-1}).
\end{align*}
In particular, the same estimate holds for every measurable choice $a^*\in\mathcal M_p$.
\end{thm}

The final result concerns the geometry of the spherical posterior. It gives sharp exponential-order concentration on a ridge band determined by the ridge estimator. This provides an all-temperature counterpart to the spherical-cap localization of Qiu and Sen~\cite{qiu2023tap}. Set
\begin{align}
S:=X^\top X,\qquad A_\Delta:=S+\Delta I_p,\qquad a_\Delta:=A_\Delta^{-1}X^\top Y. \label{eq:ridge_estimator_definition}
\end{align}
For $\varepsilon>0$, define
\begin{align}
B_\varepsilon(a_\Delta):=\left\{\beta\in S^{p-1}(\sqrt p):\left|\frac{1}{p}\langle \beta-a_\Delta,a_\Delta\rangle\right|\le \varepsilon\right\}. \label{eq:ridge_band_definition}
\end{align}

\begin{thm}[Sharp exponential-order concentration on ridge bands]\label{thm3}
Under the assumptions of Theorem~\ref{thm_sectionA_1}, there exist constants $c,\varepsilon_0>0$, depending only on $(\alpha,\Delta,C_X)$, and a nonnegative random sequence $\mathcal R_p^{\mathrm{band}}=O_P(1)$ such that, simultaneously for all $0<\varepsilon\le\varepsilon_0$,
\begin{align}
\log \Pi_S\left(B_\varepsilon(a_\Delta)^c\mid X,Y\right) \le-cp\varepsilon^2+\mathcal R_p^{\mathrm{band}}. \label{eq:ridge_band_uniform_upper}
\end{align}
Consequently, for every deterministic sequence $\varepsilon_p\in(0,\varepsilon_0]$ satisfying $\sqrt p\,\varepsilon_p\to\infty$,
\begin{align}
\Pi_S\left(B_{\varepsilon_p}(a_\Delta)\middle|X,Y\right) \xrightarrow{\mathbb P}1. \label{eq:ridge_band_shrinking_concentration}
\end{align}

Set $q_\star:=1-E_\Delta\in(0,1)$. For every fixed $0<\varepsilon<q_\star+\sqrt{q_\star}$, there are constants $0<c_\varepsilon\le C_\varepsilon<\infty$, depending only on $(\varepsilon,\alpha,\Delta,C_X)$, such that
\begin{align}
\mathbb P\left( c_\varepsilon \le-\frac1p\log \Pi_S\left(B_\varepsilon(a_\Delta)^c\mid X,Y\right) \le C_\varepsilon \right)\longrightarrow1. \label{eq:ridge_band_sharp_exponential_order}
\end{align}
\end{thm}

The identity $\|\beta\|^2=p$ gives the equivalent representation
\begin{align*}
B_\varepsilon(a_\Delta) =\left\{\beta\in S^{p-1}(\sqrt p): \left|\frac{\|\beta-a_\Delta\|^2}{p} -\left(1-\frac{\|a_\Delta\|^2}{p}\right)\right| \le2\varepsilon\right\}.
\end{align*}
By \eqref{eq:ridge_norm_rate} and \eqref{eq:deterministic_aspect_ratio_stability}, the center $1-\|a_\Delta\|^2/p$ converges in probability to $E_\Delta>0$. Thus \eqref{eq:ridge_band_shrinking_concentration} describes posterior concentration in a shell whose normalized squared radius around the ridge estimator converges to $E_\Delta>0$, rather than contraction to the ridge estimator itself.

\subsection{Future directions}

It remains open whether the direct $O_P(p^{-1})$ TAP error is sharp. The limiting behavior of $p(F_p^S-F_{\mathrm{TAP}})$ is also unknown. Establishing such a limit would require stronger assumptions than the order bounds in Assumption~\ref{assump:subgaussian_design}.

For the ridge band, Theorem~\ref{thm3} gives contraction when $\varepsilon_p\gg p^{-1/2}$ but determines neither the critical shrinking width nor the exact fixed-width exponential rate. Conditional fluctuation limits for the ridge coordinate, together with joint local large-deviation estimates for this coordinate and the Gaussian posterior radius, may address these questions.

Extensions to non-Marchenko--Pastur spectra, nonspherical priors, and nonquadratic likelihoods raise further questions. Spectrum-dependent TAP predictions \cite{maillard2019high} and rigorous high-temperature results for rotationally invariant designs with i.i.d.\ signal priors \cite{li2024random} provide starting points. The challenge is to obtain analogous quantitative all-temperature conclusions using alternatives to the Gaussian-conditioning and ridge-surrogate arguments where these methods no longer apply.

\section{TAP Accuracy Below the Fluctuation Scale}
This section proves \eqref{eq:main_tap_accuracy}--\eqref{eq:main_tap_centering} in Theorem~\ref{thm_sectionA_1}; the sharpness bounds \eqref{eq:main_spherical_sharpness}--\eqref{eq:main_tap_sharpness} are proved in Appendix~\ref{app:fluctuation_sharpness}. We first compare the spherical-prior and Gaussian-prior partition functions for the same data through an exact radial identity. We then compare the Gaussian-prior free energy and the TAP maximum with the maximum of a common quadratic ridge surrogate.

We work throughout under the model, notation, and assumptions of Theorem~\ref{thm_sectionA_1}, including the aspect-ratio condition \eqref{eq:aspect_ratio_rate}. In particular, $F_p^S$, $f_{\mathrm{TAP}}$, $F_{\mathrm{TAP}}$, and $\mathcal M_p$ are those defined in \eqref{eq:spherical_free_energy_definition}, \eqref{eq:tap_functional_definition}, and \eqref{eq:tap_optimum_definition}.

For $\gamma>0$, define
\begin{align}
L(\gamma,\Delta) &:=-\bigl(1-E(\gamma,\Delta)\bigr) +\gamma\log\left(1+\frac{E(\gamma,\Delta)}{\gamma\Delta}\right) -\log E(\gamma,\Delta). \label{eq:L_definition}
\end{align}
The quantity $E(\gamma,\Delta)$ defined above is also characterized as the unique solution $e\in(0,1)$ of
\begin{align}
e=\frac{\gamma\Delta+e}{\gamma\Delta+e+\gamma}. \label{eq:E_fixed_point}
\end{align}
For fixed $\Delta>0$, the functions $E(\cdot,\Delta)$, $L(\cdot,\Delta)$, and $\phi(\cdot,\Delta)$ are continuously differentiable on $(0,\infty)$. Since $\alpha_p\to\alpha>0$, both $\alpha_p$ and $\alpha$ lie in $[\alpha/2,2\alpha]$ for all sufficiently large $p$. Their derivatives are bounded on this interval, so the mean value theorem and \eqref{eq:aspect_ratio_rate} give a deterministic constant $C=C(\alpha,\Delta)<\infty$ such that, for all sufficiently large $p$,
\begin{align}
&|E(\alpha_p,\Delta)-E(\alpha,\Delta)| +|L(\alpha_p,\Delta)-L(\alpha,\Delta)|\nonumber\\
&\quad+|\phi(\alpha_p,\Delta)-\phi(\alpha,\Delta)| \le C|\alpha_p-\alpha| =O\left(p^{-1/2}\right). \label{eq:deterministic_aspect_ratio_stability}
\end{align}

\subsection{Gaussian-Spherical Free-Energy Comparison}

To compare the two free energies, we introduce a Gaussian-prior reference measure while keeping the data $(X,Y)$ generated by the spherical model. Let $\gamma_p:=N(0,I_p)$, and define the corresponding partition function, free energy, and posterior by
\begin{align}
Z_p^G &:=\int_{\mathbb R^p} \exp\left\{-\frac{1}{2\Delta}\|Y-X\beta\|^2\right\}d\gamma_p(\beta), \qquad F_p^G:=\frac1p\log Z_p^G, \label{eq:gaussian_partition_definition}\\
P_G(d\beta\mid X,Y) &:=\frac1{Z_p^G} \exp\left\{-\frac{1}{2\Delta}\|Y-X\beta\|^2\right\}d\gamma_p(\beta). \label{eq:gaussian_posterior_definition}
\end{align}

\begin{lem}[Mean squared radius of the Gaussian-prior posterior]
\label{lem:gaussian_posterior_typical_radius}
Let
\begin{align}
\Sigma_G:=\Delta A_\Delta^{-1} =\left(I_p+\Delta^{-1}X^\top X\right)^{-1}. \label{eq:gaussian_posterior_covariance}
\end{align}
For every fixed realization of $(X,Y)$, if $Z$ is drawn from the posterior $P_G(\cdot\mid X,Y)$, then
\begin{align}
Z\mid X,Y&\sim N(a_\Delta,\Sigma_G), \label{eq:gaussian_posterior_law}
\end{align}
and the Gaussian-prior posterior mean is $m_G=a_\Delta$. Under the joint law of the planted spherical model, namely $Y=X\beta^\star+\sqrt\Delta z$, where $\beta^\star$ is uniform on $S^{p-1}(\sqrt p)$ and $z\sim N(0,I_n)$, with $X$, $\beta^\star$, and $z$ mutually independent, we furthermore have
\begin{align}
\|a_\Delta\|&=O_P(\sqrt p), \qquad \mathbb E[\|Z\|^2\mid X,Y] = \|a_\Delta\|^2+\operatorname{Tr}(\Sigma_G) =p+O_P(\sqrt p). \label{eq:gaussian_posterior_typical_radius}
\end{align}
\end{lem}

\begin{proof}
Completing the square in the exponent of the Gaussian-prior posterior density establishes \eqref{eq:gaussian_posterior_law} for every fixed $(X,Y)$. For the probabilistic estimates, use the planted representation above and let $B:=I_p-\Sigma_G$. Then
\begin{align*}
a_\Delta=B\beta^\star+\eta, \qquad \eta:=\sqrt\Delta A_\Delta^{-1}X^\top z.
\end{align*}
Conditional on $X$, the vector $\eta$ is independent of $\beta^\star$ and satisfies
\begin{align*}
\eta\sim N(0,\Sigma_G-\Sigma_G^2).
\end{align*}
Since $\mathbb E[\eta\mid X]=0$ and $\mathbb E[\beta^\star(\beta^\star)^\top\mid X]=I_p$, conditional independence implies
\begin{align*}
\mathbb E[\|a_\Delta\|^2\mid X] =\operatorname{Tr}(B^2)+\operatorname{Tr}(\Sigma_G-\Sigma_G^2) =\operatorname{Tr}(B).
\end{align*}
The squared norm can be written as
\begin{align*}
\|a_\Delta\|^2 =(\beta^\star)^\top B^2\beta^\star +2(\beta^\star)^\top B\eta+\|\eta\|^2.
\end{align*}
Since $\beta^\star$ and $\eta$ are independent given $X$ and $\eta$ is conditionally centered Gaussian, the three terms are pairwise uncorrelated given $X$, so
\begin{align*}
\operatorname{Var}(\|a_\Delta\|^2\mid X) &=\operatorname{Var}\bigl((\beta^\star)^\top B^2\beta^\star\mid X\bigr) +4\operatorname{Tr}(B^3\Sigma_G) +2\operatorname{Tr}\bigl((\Sigma_G-\Sigma_G^2)^2\bigr).
\end{align*}
The spherical quadratic-form variance satisfies
\begin{align*}
\operatorname{Var}\bigl((\beta^\star)^\top B^2\beta^\star\mid X\bigr) =\frac{2p}{p+2}\left\{\operatorname{Tr}(B^4) -\frac{(\operatorname{Tr}(B^2))^2}{p}\right\}\le2p.
\end{align*}
Since $B=I_p-\Sigma_G$ and $0\preceq B,\Sigma_G\preceq I_p$, the remaining two terms are bounded by $4p$ and $2p$, respectively. Thus $\operatorname{Var}(\|a_\Delta\|^2\mid X)\le8p$.

By Chebyshev's inequality, applied conditionally on $X$ and then averaged over $X$,
\begin{align*}
\|a_\Delta\|^2-\operatorname{Tr}(B)=O_P(\sqrt p).
\end{align*}
Since $\operatorname{Tr}(B)\le p$, the norm bound in \eqref{eq:gaussian_posterior_typical_radius} follows. The mean-radius bound in the same equation follows from
\begin{align*}
\mathbb E[\|Z\|^2\mid X,Y]-p =\|a_\Delta\|^2-\operatorname{Tr}(B),
\end{align*}
using the Gaussian posterior formula and $B=I_p-\Sigma_G$.
\end{proof}

The next lemma compares the spherical and Gaussian-prior free energies through an exact radial identity. The proof combines \eqref{eq:gaussian_posterior_typical_radius} with the local density bounds in Lemma~\ref{lem_gaussian_typical_sphere}, which is proved in Appendix~\ref{app:gaussian_conditioning} and is independent of the linear model.

\begin{lem}
\label{lem_sectionA_1}
Under the assumptions of Theorem~\ref{thm_sectionA_1},
\begin{align*}
F_p^S-F_p^G=O_P(p^{-1}).
\end{align*}
\end{lem}

\begin{proof}
For $t>0$, let $\pi_t$ be the uniform probability measure on $S^{p-1}(\sqrt t)$ and define
\begin{align*}
Z_p^S(t) := \int_{S^{p-1}(\sqrt t)} \exp\left\{-\frac{1}{2\Delta}\|Y-X\beta\|^2\right\}d\pi_t(\beta).
\end{align*}
Let $\rho_{X,Y}$ denote the continuous density of $\|Z\|^2$ conditional on $(X,Y)$, where $Z$ is drawn from the Gaussian-prior posterior $P_G(\cdot\mid X,Y)=N(a_\Delta,\Sigma_G)$ in \eqref{eq:gaussian_posterior_law}. Under the Gaussian prior $\gamma_p=N(0,I_p)$, the squared radius $\|\beta\|^2$ has a chi-square distribution with $p$ degrees of freedom, whose density we denote by $f_{\chi_p^2}$. Moreover, conditional on $\|\beta\|^2=t$, the distribution of $\beta$ is $\pi_t$. Reweighting by the likelihood and normalizing by $Z_p^G$ therefore yields, for every $t>0$,
\begin{align}
\rho_{X,Y}(t) = f_{\chi_p^2}(t)\frac{Z_p^S(t)}{Z_p^G}. \label{eq:exact_radial_density_identity}
\end{align}
Since $Z_p^S(p)=Z_p^S$, \eqref{eq:exact_radial_density_identity} yields
\begin{align}
p(F_p^S-F_p^G) = \log\rho_{X,Y}(p)-\log f_{\chi_p^2}(p). \label{eq:exact_spherical_gaussian_identity}
\end{align}

It remains to estimate the two densities at $p$. By Lemma~\ref{lem:gaussian_posterior_typical_radius}, the posterior satisfies \eqref{eq:gaussian_posterior_typical_radius}. Set $c_-:=\Delta/(\Delta+C_X^2)$. The operator-norm bound in Assumption~\ref{assump:subgaussian_design}, together with \eqref{eq:gaussian_posterior_typical_radius}, implies that, for every $\eta>0$, there exist deterministic constants $K_0,M>0$ such that, for all sufficiently large $p$, the following bounds hold simultaneously with probability at least $1-\eta$:
\begin{align*}
c_-I_p\preceq\Sigma_G\preceq I_p,\qquad \|a_\Delta\|\le K_0\sqrt p,\qquad \left|\|a_\Delta\|^2+\operatorname{Tr}(\Sigma_G)-p\right|\le M\sqrt p.
\end{align*}
On this event, the local density bound \eqref{eq:typical_gaussian_radius_density}, proved later in Lemma~\ref{lem_gaussian_typical_sphere}, applies with $t=p$ and gives
\begin{align*}
c_M\le\sqrt p\,\rho_{X,Y}(p)\le C_M,
\end{align*}
where $c_M,C_M>0$ depend only on $(c_-,1,K_0,M)$. Since $\eta$ is arbitrary,
\begin{align}
\log\bigl(\sqrt p\,\rho_{X,Y}(p)\bigr)=O_P(1). \label{eq:posterior_radius_density_order}
\end{align}
Using Stirling's formula $\Gamma(x)=\sqrt{2\pi}x^{x-1/2}e^{-x}(1+O(x^{-1}))$ at $x=p/2$, the chi-square density at its mean satisfies
\begin{align*}
f_{\chi_p^2}(p) &=\frac{p^{p/2-1}e^{-p/2}}{2^{p/2}\Gamma(p/2)} =\frac{1+O(p^{-1})}{\sqrt{4\pi p}}.
\end{align*}
Thus $\log f_{\chi_p^2}(p)=-\frac12\log p-\frac12\log(4\pi)+O(p^{-1})$. Together with \eqref{eq:posterior_radius_density_order}, substitution into \eqref{eq:exact_spherical_gaussian_identity} gives $p(F_p^S-F_p^G)=O_P(1)$, which proves the claim.
\end{proof}

\subsection{A Common Quadratic Ridge Surrogate}

The key idea of this subsection is to control the generally nonconcave TAP functional by the strictly concave quadratic ridge functional $Q_p$ defined below. Concavity of the TAP functional cannot be assumed: Appendix~A of our earlier conference paper \cite{yu2025proof} gives an explicit example of nonconcavity already for the spherical prior. The pointwise decomposition of $f_{\mathrm{TAP}}$ into $Q_p$ minus a nonnegative scalar entropy gap establishes $Q_p$ as a global upper bound, while the ridge estimator nearly saturates this bound. This provides the quantitative control of the global TAP optimum needed below.

The following random-matrix estimates are the only inputs concerning the distribution of the design. The deterministic terms in \eqref{eq:universal_logdet} and \eqref{eq:universal_resolvent} below retain the actual aspect ratio $\alpha_p=n/p$, rather than its limit $\alpha$. In particular, retaining $\alpha_p$ in \eqref{eq:universal_logdet} preserves the $O_P(p^{-1})$ remainder needed for the comparison below the fluctuation scale.

\begin{lem}[Spectral estimates and verification for i.i.d.\ designs]
\label{lem_sectionA_2}
Suppose that $\alpha_p\to\alpha\in(0,\infty)$ and that $\Delta>0$ is fixed. Under Assumption~\ref{assump:subgaussian_design},
\begin{align}
\mathbb P\left(\|X\|_{\mathrm{op}}\le C_X\right)&\longrightarrow1, \label{eq:universal_op_norm}\\
\frac1p\log\det\left(I_p+\frac1\Delta X^\top X\right) &=L(\alpha_p,\Delta)+O_P(p^{-1}), \label{eq:universal_logdet}\\
\frac{\Delta}{p}\operatorname{Tr}(X^\top X+\Delta I_p)^{-1} &=E_p+O_P(p^{-1/2}). \label{eq:universal_resolvent}
\end{align}
Moreover, Assumption~\ref{assump:subgaussian_design} holds when $X_{ij}=\xi_{ij}/\sqrt n$, where the $\xi_{ij}$ are i.i.d.\ copies of a fixed real random variable $\xi$ satisfying
\begin{align*}
\mathbb E\xi=0,\qquad \mathbb E\xi^2=1,\qquad \mathbb E\xi^4<\infty,
\end{align*}
and $X$ is independent of $(\beta^\star,W)$.
\end{lem}

The proof is deferred to Appendix~\ref{app:tap_estimates}. For the comparison that follows, define
\begin{align}
C_p&:=-\frac12L(\alpha_p,\Delta), \label{eq:ridge_constant_definition}\\
Q_p(a)&:=-\frac{1}{2\Delta p}\|Y-Xa\|^2 -\frac{1}{2p}\|a\|^2+C_p, \label{eq:ridge_functional_definition}\\
R_p&:=\sup_{a\in\mathbb R^p}Q_p(a). \label{eq:ridge_surrogate_definition}
\end{align}
By strict concavity, $Q_p$ is uniquely maximized by the ridge estimator $a_\Delta$ defined in \eqref{eq:ridge_estimator_definition}. The next lemma compares the Gaussian-prior free energy with $R_p$ and evaluates $R_p$ to the required accuracy.

\begin{lem}
\label{lem_sectionA_3}
Under the assumptions of Theorem~\ref{thm_sectionA_1},
\begin{align*}
F_p^G-R_p&=O_P(p^{-1}),\\
R_p&=\phi(\alpha_p,\Delta)+O_P(p^{-1/2}).
\end{align*}
\end{lem}

\begin{proof}
Completing the square in the Gaussian integral and maximizing $Q_p$ yield, respectively,
\begin{align}
F_p^G &=-\frac{1}{2p}\log\det\left(I_p+\frac1\Delta X^\top X\right) -\frac{1}{2p}Y^\top(\Delta I_n+XX^\top)^{-1}Y,\nonumber\\
R_p &=C_p-\frac{1}{2p}Y^\top(\Delta I_n+XX^\top)^{-1}Y. \label{eq:ridge_optimum_exact}
\end{align}
Subtracting the two identities and applying \eqref{eq:universal_logdet}, we obtain
\begin{align}
F_p^G-R_p =-\frac{1}{2p} \left\{\log\det\left(I_p+\frac1\Delta X^\top X\right) -pL(\alpha_p,\Delta)\right\} =O_P(p^{-1}). \label{eq:gaussian_ridge_second_order}
\end{align}

To prove the ridge-centering estimate \eqref{eq:ridge_centering_rate} below, set $D_\Delta:=\Delta I_n+XX^\top$. Since $Y=X\beta^\star+W$, taking conditional expectation given $X$ yields
\begin{align*}
\mathbb E[Y^\top D_\Delta^{-1}Y\mid X] &=\operatorname{Tr}(X^\top D_\Delta^{-1}X) +\Delta\operatorname{Tr}(D_\Delta^{-1})=n.
\end{align*}
Set $A:=X^\top D_\Delta^{-1}X$. Expanding the quadratic form and using the independence of $\beta^\star$ and $W$, the three centered terms are pairwise uncorrelated given $X$. Hence
\begin{align*}
\operatorname{Var}(Y^\top D_\Delta^{-1}Y\mid X) &=\operatorname{Var}((\beta^\star)^\top A\beta^\star\mid X)\\
&\quad+4\Delta\operatorname{Tr}(X^\top D_\Delta^{-2}X) +2\Delta^2\operatorname{Tr}(D_\Delta^{-2}).
\end{align*}
Since the eigenvalues of $A$ lie in $[0,1]$, the spherical quadratic-form variance formula bounds the first term by $2p$. The inequalities $4\Delta\lambda/(\Delta+\lambda)^2\le1$ and $\Delta^2/(\Delta+\lambda)^2\le1$, for $\lambda\ge0$, bound the other two terms by $p$ and $2n$, respectively. Thus the conditional variance is at most $3p+2n=O(p)$, uniformly in $X$. Conditional Chebyshev's inequality, followed by averaging over $X$, therefore yields
\begin{align*}
Y^\top D_\Delta^{-1}Y=n+O_P(\sqrt p).
\end{align*}
Consequently,
\begin{align}
R_p &=-\frac12L(\alpha_p,\Delta)-\frac{\alpha_p}{2}+O_P(p^{-1/2})\nonumber\\
&=\phi(\alpha_p,\Delta)+O_P(p^{-1/2}), \label{eq:ridge_centering_rate}
\end{align}
where the last identity follows directly from the definitions of $L$ and $\phi$.
\end{proof}

We next record two estimates for the ridge estimator that will be used both in the TAP comparison below and in the posterior-geometry arguments.

\begin{lem}
\label{lem_sectionA_4}
Under the assumptions of Theorem~\ref{thm_sectionA_1},
\begin{align}
q_\Delta:=\frac1p\|a_\Delta\|^2 &=1-E_p+O_P(p^{-1/2}), \label{eq:ridge_norm_rate}\\
\frac1p\|Y-Xa_\Delta\|^2 &=\Delta E_p+\Delta(\alpha_p-1)+O_P(p^{-1/2}). \label{eq:ridge_residual_rate}
\end{align}
Here $E_p$ is defined in \eqref{eq:Ep_definition}.
\end{lem}

The proof is deferred to Appendix~\ref{app:tap_estimates}. To compare $Q_p$ with the TAP functional, we isolate their scalar entropy gap. For $0\le q<1$, define
\begin{align}
h_p(q):=-\frac q2+C_p +\frac{\alpha_p}{2}\log\left(1+\frac{1-q}{\alpha_p\Delta}\right) -\frac12\log(1-q). \label{eq:entropy_gap_definition}
\end{align}
For $\|a\|<\sqrt p$,
\begin{align}
h_p\left(\frac{\|a\|^2}{p}\right)=Q_p(a)-f_{\mathrm{TAP}}(a). \label{eq:tap_entropy_gap_identity}
\end{align}
The next lemma uses \eqref{eq:tap_entropy_gap_identity} and \eqref{eq:ridge_norm_rate} to compare their global optima.

\begin{lem}[Quantitative ridge approximation]
\label{lem_sectionA_5}
Under the assumptions of Theorem~\ref{thm_sectionA_1},
\begin{align}
0\le R_p-F_{\mathrm{TAP}}=O_P(p^{-1}). \label{eq:ridge_tap_second_order}
\end{align}
\end{lem}

\begin{proof}
Let $q_{0,p}:=1-E_p$, where $E_p$ is defined in \eqref{eq:Ep_definition}. The fixed-point equation \eqref{eq:E_fixed_point}, together with \eqref{eq:ridge_constant_definition}, implies that
\begin{align*}
h_p(q_{0,p})=0.
\end{align*}
Moreover,
\begin{align*}
h_p'(q) &=-\frac12-\frac{\alpha_p}{2(\alpha_p\Delta+1-q)} +\frac{1}{2(1-q)}\\
&=\frac{(q-q_{0,p})(1-q+\alpha_p\Delta/E_p)} {2(1-q)(\alpha_p\Delta+1-q)},
\end{align*}
where the second equality follows from \eqref{eq:E_fixed_point}. The factors other than $q-q_{0,p}$ are positive for $0\le q<1$. Thus $h_p'(q)<0$ for $q<q_{0,p}$ and $h_p'(q)>0$ for $q>q_{0,p}$, so
\begin{align}
h_p(q)\ge0,\qquad 0\le q<1. \label{eq_entropy_gap_nonnegative}
\end{align}
By \eqref{eq:tap_entropy_gap_identity} and \eqref{eq_entropy_gap_nonnegative}, $f_{\mathrm{TAP}}(a)\le Q_p(a)\le R_p$ for $\|a\|<\sqrt p$, and hence $F_{\mathrm{TAP}}\le R_p$.

By \eqref{eq:ridge_norm_rate},
\begin{align*}
q_\Delta=q_{0,p}+O_P(p^{-1/2}).
\end{align*}
Since $E_p\to E_\Delta\in(0,1)$, there exists a fixed $\delta>0$ such that the intervals $[q_{0,p}-\delta,q_{0,p}+\delta]$ lie in one fixed compact subset of $(0,1)$ for all sufficiently large $p$. Since
\begin{align*}
h_p''(q)=\frac{1}{2(1-q)^2} -\frac{\alpha_p}{2(\alpha_p\Delta+1-q)^2},
\end{align*}
the second derivatives are uniformly bounded on these intervals. The event $\{|q_\Delta-q_{0,p}|\le\delta\}$ has probability tending to one. On this event, $a_\Delta$ lies in the domain of $f_{\mathrm{TAP}}$, and its optimality for $Q_p$ gives
\begin{align*}
F_{\mathrm{TAP}} \ge f_{\mathrm{TAP}}(a_\Delta) =R_p-h_p(q_\Delta).
\end{align*}
Since $h_p(q_{0,p})=h_p'(q_{0,p})=0$, Taylor's theorem therefore yields a deterministic constant $C<\infty$ such that, on the same event,
\begin{align*}
0\le R_p-F_{\mathrm{TAP}} \le h_p(q_\Delta)\le C(q_\Delta-q_{0,p})^2.
\end{align*}
The quantity $p(q_\Delta-q_{0,p})^2$ is bounded in probability, and the exceptional event has probability tending to zero. This proves \eqref{eq:ridge_tap_second_order}.
\end{proof}

\begin{rem}
\label{rem:tap_sharpness}
The sharpness bounds \eqref{eq:main_spherical_sharpness} and \eqref{eq:main_tap_sharpness} concern the separate fluctuations of $F_p^S$ and $F_{\mathrm{TAP}}$ around $\phi(\alpha_p,\Delta)$. They do not provide a matching lower bound for the direct approximation error $|F_p^S-F_{\mathrm{TAP}}|$, for which we prove only the upper bound \eqref{eq:main_tap_accuracy}.
\end{rem}

\begin{proof}[Proof of the upper bounds in Theorem~\ref{thm_sectionA_1}]
By Lemmas~\ref{lem_sectionA_1}, \ref{lem_sectionA_3}, and \ref{lem_sectionA_5},
\begin{align*}
F_p^S-F_{\mathrm{TAP}} &=(F_p^S-F_p^G)+(F_p^G-R_p)+(R_p-F_{\mathrm{TAP}})\\
&=O_P(p^{-1}).
\end{align*}
This proves \eqref{eq:main_tap_accuracy}. The ridge-centering estimate \eqref{eq:ridge_centering_rate}, together with the $O_P(p^{-1})$ comparisons in Lemmas~\ref{lem_sectionA_1}, \ref{lem_sectionA_3}, and \ref{lem_sectionA_5}, gives \eqref{eq:main_spherical_centering} and \eqref{eq:main_tap_centering}. Finally, \eqref{eq:deterministic_aspect_ratio_stability} gives the corresponding bounds around $\phi(\alpha,\Delta)$.
\end{proof}

\section{TAP Maximizers and the Spherical Posterior Mean}

This section proves Theorem~\ref{thm_sectionA_2} by comparing the spherical posterior mean and all global TAP maximizers with the Gaussian posterior mean $m_G=a_\Delta$. The argument uses the entropy-gap and ridge-approximation estimates from the preceding section, but not the free-energy results of Theorem~\ref{thm_sectionA_1}.

The comparison of the two posterior means requires controlling how the mean of a high-dimensional Gaussian distribution changes when its squared norm is fixed. Since conditioning on a fixed value of $\|Z\|_2^2$ is a probability-zero event, a precise continuous version of the conditional mean is needed. The following lemma defines this version and bounds its deviation from the original Gaussian mean when the covariance eigenvalues are uniformly bounded above and away from zero. The result is independent of the linear model and will be applied to the Gaussian posterior.

\begin{lem}[Gaussian conditioning near the typical radius]
\label{lem_gaussian_typical_sphere}
Fix $0<c_-<c_+<\infty$ and $K<\infty$, and let $Z\sim N(a,\Sigma_p)$ in $\mathbb R^p$ satisfy
\begin{align*}
c_-I_p\preceq \Sigma_p\preceq c_+I_p, \qquad \|a\|_2\le K\sqrt p.
\end{align*}
Set $\mu_p:=\mathbb E\|Z\|_2^2=\|a\|_2^2+\operatorname{Tr}(\Sigma_p)$. For $r_p>0$, the conditional mean is understood in the continuous coarea sense:
\begin{align*}
\mathbb E\left[Z\mid \|Z\|_2^2=r_p^2\right] := \frac{\displaystyle\int_{\|z\|=r_p}z f_Z(z)\,d\mathcal H^{p-1}(z)} {\displaystyle\int_{\|z\|=r_p}f_Z(z)\,d\mathcal H^{p-1}(z)},
\end{align*}
where $f_Z$ is the Gaussian density and $\mathcal H^{p-1}$ is surface measure. There exist $\varepsilon_0,C>0$, depending only on $(c_-,c_+,K)$, such that, for all sufficiently large $p$ and every $r_p>0$ satisfying $|r_p^2-\mu_p|\le\varepsilon_0p$,
\begin{align*}
\left\|\mathbb E\left[Z\mid \|Z\|_2^2=r_p^2\right]-a\right\|_2 \le C\left(1+\frac{|r_p^2-\mu_p|}{\sqrt p}\right).
\end{align*}
Moreover, if $\rho_Z$ denotes the continuous density of $\|Z\|_2^2$ on $(0,\infty)$, then for every fixed $M<\infty$ there exist $0<c_M<C_M<\infty$, depending only on $(c_-,c_+,K,M)$, such that, for all sufficiently large $p$,
\begin{align}
\frac{c_M}{\sqrt p}\le \rho_Z(t)\le\frac{C_M}{\sqrt p} \label{eq:typical_gaussian_radius_density}
\end{align}
whenever $|t-\mu_p|\le M\sqrt p$. There also exists $C_{\mathrm{up}}<\infty$, depending only on $(c_-,c_+,K)$, such that, for all sufficiently large $p$,
\begin{align}
\sup_{t>0}\rho_Z(t)\le\frac{C_{\mathrm{up}}}{\sqrt p}. \label{eq:global_gaussian_radius_density_upper}
\end{align}
The thresholds on $p$ are uniform over $a$ and $\Sigma_p$ satisfying the displayed hypotheses.
\end{lem}

The proof is deferred to Appendix~\ref{app:gaussian_conditioning}.

Returning to the linear model under the standing assumptions, including \eqref{eq:aspect_ratio_rate}, we apply Lemma~\ref{lem_gaussian_typical_sphere} to the Gaussian-prior posterior with $r_p^2=p$. Conditioning this posterior on $\|Z\|_2^2=p$ yields the spherical-prior posterior. Lemma~\ref{lem:gaussian_posterior_typical_radius} and Assumption~\ref{assump:subgaussian_design} verify the conditions required for the following comparison.

\begin{lem}
\label{lem_spherical_gaussian_mean_comparison}
Let $m_S$ be as in Theorem~\ref{thm_sectionA_2}, and let $m_G$ be the Gaussian-prior posterior mean for the same realization of $(X,Y)$. Then
\begin{align*}
\|m_G-m_S\|_2=O_P(1).
\end{align*}
\end{lem}

\begin{proof}
Lemma~\ref{lem:gaussian_posterior_typical_radius} identifies the Gaussian-prior posterior as $N(a_\Delta,\Sigma_G)$, with $m_G=a_\Delta$. By polar disintegration, its conditional law given $\|Z\|_2^2=p$ is the spherical posterior, since the Gaussian-prior factor $e^{-\|Z\|_2^2/2}$ is constant on $S^{p-1}(\sqrt p)$. Therefore,
\begin{align}
m_S=\mathbb E[Z\mid X,Y,\|Z\|_2^2=p], \label{eq_spherical_as_conditioned_gaussian}
\end{align}
where the conditional expectation is understood in the continuous coarea sense defined in Lemma~\ref{lem_gaussian_typical_sphere}. Let $\mu_p$ denote the conditional mean squared radius in \eqref{eq:gaussian_posterior_typical_radius}. This estimate, together with \eqref{eq:universal_op_norm} and \eqref{eq:gaussian_posterior_covariance}, implies that, for every $\delta>0$, there are deterministic constants $c_->0$ and $K<\infty$ such that, for all sufficiently large $p$, the following event has probability at least $1-\delta$:
\begin{align*}
c_-I_p\preceq\Sigma_G\preceq I_p, \qquad \|a_\Delta\|_2\le K\sqrt p, \qquad |\mu_p-p|\le K\sqrt p.
\end{align*}
For all sufficiently large $p$, $K\sqrt p\le\varepsilon_0p$, so Lemma~\ref{lem_gaussian_typical_sphere} applies with $r_p^2=p$ on this event. Together with \eqref{eq_spherical_as_conditioned_gaussian}, the lemma gives, on this event,
\begin{align*}
\|m_S-m_G\|_2 \le C\left(1+\frac{|\mu_p-p|}{\sqrt p}\right)\le C(1+K).
\end{align*}
Since $\delta>0$ is arbitrary, $\|m_S-m_G\|_2=O_P(1)$.
\end{proof}

The preceding lemma shows that replacing the Gaussian prior by the spherical prior changes the posterior mean by only $O_P(1)$. Recall from \eqref{eq:tap_optimum_definition} that $\mathcal M_p$ is the set of all global maximizers of $f_{\mathrm{TAP}}$. It remains to compare these maximizers with the Gaussian posterior mean. The entropy-gap inequality bounds $f_{\mathrm{TAP}}$ above by the strictly concave ridge surrogate $Q_p$, whose unique maximizer is $a_\Delta=m_G$. Comparing the value at a global TAP maximizer with the value at $a_\Delta$ and using the quadratic curvature of $Q_p$ yields the uniform distance estimate below.

\begin{lem}
\label{lem_tap_maximizer_gaussian_mean_rate}
Under the assumptions of Theorem~\ref{thm_sectionA_2},
\begin{align*}
\sup_{a\in\mathcal M_p}\frac{1}{p}\|a-m_G\|_2^2 =O_P(p^{-1}).
\end{align*}
\end{lem}

\begin{proof}
By \eqref{eq:tap_entropy_gap_identity} and \eqref{eq_entropy_gap_nonnegative}, $f_{\mathrm{TAP}}\le Q_p$ on its finite domain. The ridge functional \eqref{eq:ridge_functional_definition} has unique maximizer $a_\Delta$, and, for every $a\in\mathbb R^p$,
\begin{align*}
Q_p(a_\Delta)-Q_p(a) =\frac{1}{2\Delta p}(a-a_\Delta)^\top A_\Delta(a-a_\Delta) \ge\frac{1}{2p}\|a-a_\Delta\|_2^2.
\end{align*}
For every $a\in\mathcal M_p$, the definition \eqref{eq:tap_optimum_definition} gives $Q_p(a)\ge f_{\mathrm{TAP}}(a) =F_{\mathrm{TAP}}$. Consequently, \eqref{eq:ridge_tap_second_order} yields
\begin{align*}
\sup_{a\in\mathcal M_p}\frac1p\|a-a_\Delta\|_2^2 \le 2(R_p-F_{\mathrm{TAP}})=O_P(p^{-1}).
\end{align*}
Since $m_G=a_\Delta$ by \eqref{eq:gaussian_posterior_law}, this proves the lemma.
\end{proof}

\begin{proof}[Proof of Theorem~\ref{thm_sectionA_2}]
The squared triangle inequality and Lemmas~\ref{lem_tap_maximizer_gaussian_mean_rate} and~\ref{lem_spherical_gaussian_mean_comparison} yield
\begin{align*}
\sup_{a\in\mathcal M_p}\frac1p\|a-m_S\|_2^2 &\le 2\sup_{a\in\mathcal M_p}\frac1p\|a-m_G\|_2^2 +\frac2p\|m_G-m_S\|_2^2\\
&=O_P(p^{-1}).
\end{align*}
\end{proof}

\section{Sharp Exponential-Order Concentration on Ridge Bands}

This section proves Theorem~\ref{thm3} under the model assumptions, including \eqref{eq:aspect_ratio_rate}. Recall the ridge estimator $a_\Delta$ and the ridge band $B_\varepsilon(a_\Delta)$, defined in \eqref{eq:ridge_estimator_definition} and \eqref{eq:ridge_band_definition}, respectively. The upper bound \eqref{eq:ridge_band_uniform_upper} uses the exact representation of the spherical posterior as a Gaussian posterior conditioned on its squared radius. To prove the upper bound on the logarithmic rate in \eqref{eq:ridge_band_sharp_exponential_order}, we construct a spherical cap contained in the complement of the ridge band.

\begin{proof}[Proof of Theorem~\ref{thm3}]
We condition throughout on $(X,Y)$ and abbreviate
\begin{align*}
a&:=a_\Delta, & \Sigma&:=\Sigma_G.
\end{align*}
By \eqref{eq:gaussian_posterior_law}, the Gaussian-prior posterior is the law of $Z\sim N(a,\Sigma)$. The polar-disintegration argument used to derive \eqref{eq_spherical_as_conditioned_gaussian} shows that the spherical posterior is the coarea conditional law of $Z$ given $\|Z\|^2=p$. Define
\begin{align*}
L(Z)&:=\langle Z-a,a\rangle, & R(Z)&:=\|Z\|^2, & v_p&:=a^\top\Sigma a.
\end{align*}
Thus
\begin{align}
\Pi_S\left(B_\varepsilon(a)^c\mid X,Y\right) =\mathbb P\left(\left.|L(Z)|>p\varepsilon\,\right|R(Z)=p,X,Y\right), \label{eq:ridge_band_conditioned_gaussian}
\end{align}
where the conditional probability is understood in the coarea sense used in Lemma~\ref{lem_gaussian_typical_sphere}.

For $\tau\in\{-1,1\}$ and $s\ge0$, introduce the linear exponential tilt
\begin{align}
\frac{dP_{\tau,s}}{dP_0}(z) :=\exp\left\{\tau sL(z)-\frac{s^2v_p}{2}\right\}, \label{eq:ridge_coordinate_linear_tilt}
\end{align}
where $P_0:=N(a,\Sigma)$ is the Gaussian-prior posterior conditional on the fixed data $(X,Y)$. Completing the square gives
\begin{align}
P_{\tau,s}=N(a+\tau s\Sigma a,\Sigma). \label{eq:ridge_coordinate_tilted_gaussian}
\end{align}
Let $\rho_{\tau,s}$ denote the density of $R(Z)$ under $P_{\tau,s}$, and write $\rho_{0,0}$ for the corresponding density under $P_0$, always using the continuous versions. Writing $f_0$ for the Lebesgue density of $P_0$, define, for $r>0$,
\begin{align*}
h_{\tau,\varepsilon}(r) :=\int_{\|z\|^2=r} \frac{\mathbf1_{\{\tau L(z)>p\varepsilon\}}f_0(z)}{2\|z\|} \,d\mathcal H^{p-1}(z).
\end{align*}
Here $\mathcal H^{p-1}$ denotes $(p-1)$-dimensional Hausdorff measure, which on the sphere $\{z:\|z\|^2=r\}$ is the usual unnormalized surface measure. The coarea formula identifies this as a density of the subprobability measure $P_0(R(Z)\in dr,\ \tau L(Z)>p\varepsilon\mid X,Y)$.

The next step is a change-of-measure argument for binary hypothesis testing, conditional on $(X,Y)$. For $s>0$, the event $\{\tau L(Z)>p\varepsilon\}$ is the rejection region of the likelihood-ratio test of $P_0$ against $P_{\tau,s}$ with threshold $\exp\{sp\varepsilon-s^2v_p/2\}$. We apply the corresponding pointwise density bound on the sphere before normalizing by its radial density. Let $f_{\tau,s}$ be the Lebesgue density of $P_{\tau,s}$. By \eqref{eq:ridge_coordinate_linear_tilt}, for every $z\in\mathbb R^p$,
\begin{align*}
\mathbf1_{\{\tau L(z)>p\varepsilon\}}f_0(z) &=\mathbf1_{\{\tau L(z)>p\varepsilon\}} \exp\left\{-\tau sL(z)+\frac{s^2v_p}{2}\right\}f_{\tau,s}(z)\\
&\le\exp\left\{-sp\varepsilon+\frac{s^2v_p}{2}\right\}f_{\tau,s}(z),
\end{align*}
where the inequality uses $s\ge0$. Dividing by $2\|z\|$ and integrating over the sphere $\{z:\|z\|^2=p\}$ gives
\begin{align*}
h_{\tau,\varepsilon}(p) &\le \exp\left\{-sp\varepsilon+\frac{s^2v_p}{2}\right\} \int_{\|z\|^2=p}\frac{f_{\tau,s}(z)}{2\|z\|}\,d\mathcal H^{p-1}(z)\\
&= \exp\left\{-sp\varepsilon+\frac{s^2v_p}{2}\right\} \rho_{\tau,s}(p).
\end{align*}
The last equality is the coarea formula for $R(Z)$ under $P_{\tau,s}$. Consequently,
\begin{align}
\mathbb P\left(\left.\tau L(Z)>p\varepsilon\,\right|R(Z)=p,X,Y\right) \le \exp\left\{-sp\varepsilon+\frac{s^2v_p}{2}\right\} \frac{\rho_{\tau,s}(p)}{\rho_{0,0}(p)}. \label{eq:conditioned_ridge_coordinate_tail}
\end{align}

We now make this bound uniform for small $\varepsilon$. By the high-probability operator-norm bound in Assumption~\ref{assump:subgaussian_design}, together with \eqref{eq:gaussian_posterior_covariance} and \eqref{eq:ridge_norm_rate}, there are deterministic constants $c_-,C_a,V>0$, depending only on $(\alpha,\Delta,C_X)$, such that the event
\begin{align*}
\mathcal E_p^{\mathrm{upper}} :=\left\{c_-I_p\preceq\Sigma\preceq I_p,\ \|a\|\le C_a\sqrt p,\ v_p\le Vp\right\}
\end{align*}
has probability tending to one. For example, one may take $c_-:=\Delta/[2(\Delta+C_X^2)]$, $C_a:=2$, and $V:=4$, since $\|a\|^2/p\xrightarrow{\mathbb P}q_\star<1$ and $\Sigma\preceq I_p$. Fix $\varepsilon_0:=1$. On this event, for each $0<\varepsilon\le\varepsilon_0$, set $s:=\varepsilon/V$. The mean in \eqref{eq:ridge_coordinate_tilted_gaussian} has norm at most $(1+\varepsilon_0/V)C_a\sqrt p$. Therefore, applying \eqref{eq:global_gaussian_radius_density_upper} uniformly over these choices of $\varepsilon$ and $\tau\in\{-1,1\}$ gives a deterministic $C_{\mathrm{up}}<\infty$ such that
\begin{align*}
\rho_{\tau,s}(p)\le\frac{C_{\mathrm{up}}}{\sqrt p}, \qquad \tau\in\{-1,1\}.
\end{align*}
Moreover,
\begin{align*}
-sp\varepsilon+\frac{s^2v_p}{2} \le-\frac{p\varepsilon^2}{2V}.
\end{align*}
Applying \eqref{eq:conditioned_ridge_coordinate_tail} for both signs and using \eqref{eq:posterior_radius_density_order}, we obtain, simultaneously for $0<\varepsilon\le\varepsilon_0$,
\begin{align}
\log\Pi_S\left(B_\varepsilon(a_\Delta)^c\mid X,Y\right) \le-\frac{p\varepsilon^2}{2V}+D_p \label{eq:uniform_conditioned_ridge_tail}
\end{align}
on $\mathcal E_p^{\mathrm{upper}}$, where
\begin{align*}
D_p :=\log(2C_{\mathrm{up}}) -\log\left(\sqrt p\,\rho_{0,0}(p)\right) =O_P(1).
\end{align*}
Set
\begin{align*}
c&:=\frac1{4V},\\
\mathcal R_p^{\mathrm{band}} &:=(D_p)_++cp\varepsilon_0^2 \mathbf1_{(\mathcal E_p^{\mathrm{upper}})^c}.
\end{align*}
The first term is $O_P(1)$, while the second is $o_P(1)$ because $\mathbb P((\mathcal E_p^{\mathrm{upper}})^c)\to0$. On $\mathcal E_p^{\mathrm{upper}}$, \eqref{eq:uniform_conditioned_ridge_tail} implies \eqref{eq:ridge_band_uniform_upper}. On the complementary event, the same inequality follows from $\log\Pi_S(B_\varepsilon(a_\Delta)^c\mid X,Y)\le0$ and $\varepsilon\le\varepsilon_0$. Thus \eqref{eq:ridge_band_uniform_upper} holds simultaneously for all such $\varepsilon$ and all sufficiently large $p$. Enlarging finitely many terms of $\mathcal R_p^{\mathrm{band}}$ makes the bound valid for every $p$ without affecting its order in probability. If $\sqrt p\,\varepsilon_p\to\infty$, division by $p\varepsilon_p^2$ proves the concentration statement \eqref{eq:ridge_band_shrinking_concentration}.

To complete \eqref{eq:ridge_band_sharp_exponential_order}, we now prove the posterior-mass lower bound \eqref{eq:ridge_band_cap_mass_lower} below. Write $q_p:=q_\Delta$, with $q_\Delta$ defined in \eqref{eq:ridge_norm_rate}. By \eqref{eq:ridge_norm_rate}, \eqref{eq:deterministic_aspect_ratio_stability}, and the definition of $q_\star$,
\begin{align}
q_p\xrightarrow{\mathbb P}q_\star\in(0,1). \label{eq:ridge_norm_converse_limit}
\end{align}
Fix $0<\varepsilon<q_\star+\sqrt{q_\star}$. Choose $r\in(0,1)$ such that
\begin{align*}
q_\star+r\sqrt{q_\star}>\varepsilon.
\end{align*}
By continuity, there is a $\delta\in(0,q_\star/2)$ such that
\begin{align}
q+r\sqrt q>\varepsilon \qquad\text{whenever}\qquad |q-q_\star|\le\delta. \label{eq:ridge_cap_uniform_margin}
\end{align}
On the event $\{|q_p-q_\star|\le\delta\}$, define
\begin{align*}
u_p:=\frac{a_\Delta}{\|a_\Delta\|}, \qquad \mathcal C_p(r):=\left\{ \beta\in S^{p-1}(\sqrt p): \frac{\langle\beta,u_p\rangle}{\sqrt p}\le-r \right\}.
\end{align*}
For every $\beta\in\mathcal C_p(r)$,
\begin{align*}
\frac1p\langle\beta-a_\Delta,a_\Delta\rangle &=\frac{\|a_\Delta\|}{p}\langle\beta,u_p\rangle-q_p\\
&\le-r\sqrt{q_p}-q_p< -\varepsilon,
\end{align*}
where the last inequality follows from \eqref{eq:ridge_cap_uniform_margin}. Hence
\begin{align}
\mathcal C_p(r)\subseteq B_\varepsilon(a_\Delta)^c. \label{eq:ridge_cap_in_band_complement}
\end{align}

Choose any fixed $r'\in(r,1)$. Rotational invariance and the density of the first coordinate of a uniform point on the unit sphere give
\begin{align*}
\pi\bigl(\mathcal C_p(r)\bigr) &\ge (r'-r)\frac{\Gamma(p/2)}{\sqrt\pi\,\Gamma((p-1)/2)} (1-(r')^2)^{(p-3)/2}\\
&\ge\exp\{-C_{\mathrm{cap}}p\}
\end{align*}
for all sufficiently large $p$, where $C_{\mathrm{cap}}<\infty$ depends only on $r$ and $r'$. The last inequality follows from $\Gamma(p/2)/\Gamma((p-1)/2)\asymp\sqrt p$.

Write $H_Y(\beta):=(2\Delta)^{-1}\|Y-X\beta\|^2$. Let $C_X$ be the constant from Assumption~\ref{assump:subgaussian_design}, and choose a deterministic $C_W<\infty$ such that the event
\begin{align*}
\mathcal E_p^{\mathrm{lower}} :=\left\{\|X\|_{\mathrm{op}}\le C_X,\ \|W\|\le C_W\sqrt p,\ |q_p-q_\star|\le\delta\right\}
\end{align*}
has probability tending to one. Such a $C_W$ exists because $\|W\|^2/p\xrightarrow{\mathbb P}\alpha\Delta$. On this event,
\begin{align*}
\|Y\|\le(C_X+C_W)\sqrt p,
\end{align*}
and hence, uniformly over $\beta\in S^{p-1}(\sqrt p)$,
\begin{align*}
H_Y(\beta) \le\frac{(2C_X+C_W)^2}{2\Delta}p =:C_Hp.
\end{align*}
Since $H_Y\ge0$, we also have $Z_p^S\le1$. Therefore, on $\mathcal E_p^{\mathrm{lower}}$, \eqref{eq:ridge_cap_in_band_complement} gives
\begin{align}
\Pi_S\left(B_\varepsilon(a_\Delta)^c\mid X,Y\right) &\ge \int_{\mathcal C_p(r)}e^{-H_Y(\beta)}d\pi(\beta)\nonumber\\
&\ge\exp\{-(C_H+C_{\mathrm{cap}})p\}. \label{eq:ridge_band_cap_mass_lower}
\end{align}
Applying \eqref{eq:ridge_band_uniform_upper} at $\min\{\varepsilon,\varepsilon_0\}$ gives, with probability tending to one,
\begin{align}
-\frac1p\log\Pi_S\left(B_\varepsilon(a_\Delta)^c\mid X,Y\right) \ge \frac c2\min\{\varepsilon,\varepsilon_0\}^2, \label{eq:ridge_band_rate_lower}
\end{align}
because $\mathcal R_p^{\mathrm{band}}/p\xrightarrow{\mathbb P}0$. When $\varepsilon>\varepsilon_0$, this uses $B_\varepsilon(a_\Delta)^c\subseteq B_{\varepsilon_0}(a_\Delta)^c$. Combining \eqref{eq:ridge_band_rate_lower} with \eqref{eq:ridge_band_cap_mass_lower} proves \eqref{eq:ridge_band_sharp_exponential_order}.
\end{proof}

\begin{acks}
The authors would like to thank Kevin Luo, Zhou Fan, and Subhabrata Sen for helpful discussions.
J.L.'s research was supported in part by NSF Grant DMS-2515510.
\end{acks}

\begin{acks}[Statement on AI use]
OpenAI's GPT-6 Astra, accessed through Codex, assisted with improving the presentation of the proofs written by the authors. The authors independently verified all claims and take full responsibility.
\end{acks}

\appendix
\section{Auxiliary Estimates for the TAP Free-Energy Formula}
\label{app:tap_estimates}

\begin{proof}[Proof of Lemma~\ref{lem_sectionA_2}]
The first condition in Assumption~\ref{assump:subgaussian_design} directly gives \eqref{eq:universal_op_norm}.

Let $\lambda_1,\ldots,\lambda_p$ be the eigenvalues of $X^\top X$, and define
\begin{align*}
\widehat\mu_{X,p}:=\frac1p\sum_{i=1}^p\delta_{\lambda_i}, \qquad f_\Delta(x):=\log\left(1+\frac{x}{\Delta}\right).
\end{align*}
Then
\begin{align*}
\frac1p\log\det\left(I_p+\frac1\Delta X^\top X\right) =\int f_\Delta(x) d\widehat\mu_{X,p}(x).
\end{align*}
Let $\mu_{\mathrm{MP},\alpha_p}$ denote the Marchenko-Pastur law corresponding to the aspect ratio $p/n=1/\alpha_p$. Explicitly,
\begin{align*}
\mu_{\mathrm{MP},\alpha_p} =(1-\alpha_p)_+\delta_0 +\frac{\alpha_p}{2\pi x} \sqrt{(b_{\alpha_p}-x)(x-a_{\alpha_p})} \mathbf 1_{[a_{\alpha_p},b_{\alpha_p}]}(x)\,dx,
\end{align*}
where $a_\gamma=(1-\gamma^{-1/2})^2$ and $b_\gamma=(1+\gamma^{-1/2})^2$ \cite{Bai2010}. By the second condition in Assumption~\ref{assump:subgaussian_design},
\begin{align*}
\int f_\Delta(x)\,d\widehat\mu_{X,p}(x) &=\int f_\Delta(x)\,d\mu_{\mathrm{MP},\alpha_p}(x) +O_P\left(\frac1p\right).
\end{align*}

For every fixed $\gamma>0$, the deterministic MP integral satisfies
\begin{align*}
\int \log\left(1+\frac{x}{\Delta}\right) d\mu_{\mathrm{MP},\gamma}(x)=L(\gamma,\Delta).
\end{align*}
We verify the identity for a generic parameter, denoted temporarily by $\alpha$, to simplify notation. Set
\begin{align*}
\mathcal L(\Delta):=\int \log\left(1+\frac{x}{\Delta}\right)d\mu_{\mathrm{MP},\alpha}(x).
\end{align*}
For $\Delta$ in a compact subinterval of $(0,\infty)$, the derivative of the integrand is uniformly bounded on the MP support. Differentiation under the integral therefore gives
\begin{align*}
\mathcal L'(\Delta)=\int\left(\frac1{x+\Delta}-\frac1\Delta\right)d\mu_{\mathrm{MP},\alpha}(x).
\end{align*}
For the present normalization, the positive MP Stieltjes transform
\begin{align*}
m_\alpha(-\Delta):=\int\frac1{x+\Delta}\,d\mu_{\mathrm{MP},\alpha}(x)
\end{align*}
satisfies the quadratic equation
\begin{align*}
\Delta m_\alpha(-\Delta)^2 +(\alpha\Delta+\alpha-1)m_\alpha(-\Delta)-\alpha=0
\end{align*}
\cite{Bai2010}. Its unique positive root is $E_\Delta/\Delta$, by the definition of $E(\alpha,\Delta)$. Hence
\begin{align*}
\mathcal L'(\Delta)=\frac{E_\Delta-1}{\Delta}.
\end{align*}

The fixed-point equation \eqref{eq:E_fixed_point}, with $\gamma=\alpha$ and $e=E_\Delta$, can be rearranged as
\begin{align*}
\alpha\Delta+E_\Delta=\frac{\alpha E_\Delta}{1-E_\Delta}.
\end{align*}
Differentiating \eqref{eq:L_definition} with respect to $\Delta$ yields
\begin{align*}
\partial_\Delta L(\alpha,\Delta) &=E_\Delta'+\alpha\left(\frac{\alpha+E_\Delta'}{\alpha\Delta+E_\Delta}-\frac1\Delta\right)-\frac{E_\Delta'}{E_\Delta} \nonumber\\
&=E_\Delta'\left(1+\frac{\alpha}{\alpha\Delta+E_\Delta}-\frac1{E_\Delta}\right)+\frac{\alpha^2}{\alpha\Delta+E_\Delta}-\frac{\alpha}{\Delta}.
\end{align*}
Using $\alpha/(\alpha\Delta+E_\Delta)=(1-E_\Delta)/E_\Delta$, the coefficient of $E_\Delta'$ vanishes. Thus
\begin{align*}
\partial_\Delta L(\alpha,\Delta) =\frac{\alpha^2}{\alpha\Delta+E_\Delta}-\frac{\alpha}{\Delta} =\frac{\alpha(1-E_\Delta)}{E_\Delta}-\frac{\alpha}{\Delta}.
\end{align*}
The fixed-point equation also gives
\begin{align*}
\alpha\Delta(1-E_\Delta)=E_\Delta(\alpha+E_\Delta-1),
\end{align*}
and therefore
\begin{align*}
\partial_\Delta L(\alpha,\Delta) =\frac{\alpha+E_\Delta-1}{\Delta}-\frac{\alpha}{\Delta} =\frac{E_\Delta-1}{\Delta} =\mathcal L'(\Delta).
\end{align*}
Finally, as $\Delta\to\infty$, we have $E_\Delta\to1$, and both $\mathcal L(\Delta)$ and $L(\alpha,\Delta)$ tend to zero. Therefore $\mathcal L(\Delta)=L(\alpha,\Delta)$ for every fixed $\Delta>0$. Applying this identity with $\gamma=\alpha_p$ proves \eqref{eq:universal_logdet}.

The third condition in Assumption~\ref{assump:subgaussian_design} and the Marchenko-Pastur Stieltjes-transform identity give
\begin{align}
\frac{\Delta}{p}\operatorname{Tr}(X^\top X+\Delta I_p)^{-1} &=\int\frac{\Delta}{x+\Delta} d\mu_{\mathrm{MP},\alpha_p}(x)+O_P(p^{-1/2})\nonumber\\
&=E_p+O_P(p^{-1/2}). \label{eq56}
\end{align}
This proves \eqref{eq:universal_resolvent}. The argument also applies when $\alpha=1$: the test functions and their derivatives are bounded near zero for every fixed $\Delta>0$.

It remains to verify Assumption~\ref{assump:subgaussian_design} for the i.i.d.\ designs specified in Lemma~\ref{lem_sectionA_2}. Suppose that $X_{ij}=\xi_{ij}/\sqrt n$, where the $\xi_{ij}$ are standardized i.i.d.\ random variables with a fixed distribution and finite fourth moment. By the Bai-Yin spectral-edge theorem \cite[Theorem~5.8]{Bai2010}, $\|X\|_{\mathrm{op}}\xrightarrow{\mathbb P}1+\alpha^{-1/2}$. Thus any fixed $C_X>1+\alpha^{-1/2}$ satisfies
\begin{align*}
\mathbb P\left(\|X\|_{\mathrm{op}}\le C_X\right)\longrightarrow1.
\end{align*}
Najim and Yao \cite[Theorems~2 and~3]{najim2016gaussian} study Gaussian fluctuations and deterministic bias for linear spectral statistics of large sample covariance matrices with general population covariance. Their results allow standardized i.i.d.\ entries with finite fourth moment, without requiring the fourth moment to match the Gaussian value. The non-Gaussian fourth cumulant contributes to both covariance and bias. They also replace analyticity of the test function by smoothness: Theorem~2 treats centered fluctuations for $C_c^3$ functions, while Theorem~3 gives the bias expansion for $C_c^{18}$ functions. Here we need only tightness of the unnormalized fluctuations and boundedness of the bias, not their precise Gaussian approximation. These estimates verify Assumption~\ref{assump:subgaussian_design} for the present i.i.d.\ designs; the TAP and posterior comparisons are established separately. In their notation, take $N=p$, sample size $n=n_p$, $R_n=I_p$, and $c_n=p/n=1/\alpha_p$. Their aspect-ratio condition allows any positive finite limit, including $c_n\to1/\alpha>1$ when $\alpha<1$. Fix $\delta>0$ and define
\begin{align*}
f_\delta(x):=\log\left(1+\frac{x}{\delta}\right), \qquad g_\delta(x):=\frac{\delta}{x+\delta}.
\end{align*}
For either $h=f_\delta$ or $h=g_\delta$, choose a fixed $\widetilde h\in C_c^\infty(\mathbb R)$ that agrees with $h$ on $[0,C_X^2]$. Enlarging $C_X$ if necessary, this interval also contains the support of $\mu_{\mathrm{MP},\alpha_p}$ for all sufficiently large $p$. We decompose
\begin{align*}
&\sum_{i=1}^p\widetilde h(\lambda_i) -p\int\widetilde h(x)\,d\mu_{\mathrm{MP},\alpha_p}(x)\\
&\qquad=\left\{\sum_{i=1}^p\widetilde h(\lambda_i) -\mathbb E\sum_{i=1}^p\widetilde h(\lambda_i)\right\}\\
&\qquad\quad+\left\{\mathbb E\sum_{i=1}^p\widetilde h(\lambda_i) -p\int\widetilde h(x)\,d\mu_{\mathrm{MP},\alpha_p}(x)\right\}.
\end{align*}
Theorem~2 of Najim and Yao shows that the first term is $O_P(1)$. Their Theorem~3, which requires a $C_c^{18}$ test function, expresses the second term as $Z_n^2(\widetilde h)+o(1)$. To bound this bias, use their resolvent bias $\mathcal B_n$ and smooth extension $\Phi_{17}(\widetilde h)$ from (6), with a fixed smooth cutoff equal to one for $|y|\le1$ and compactly supported. Their (101) and the definition of $\bar\partial:=\partial_x+i\partial_y$ give
\begin{align*}
|\mathcal B_n(x+iy)|&\le K|x+iy|^3y^{-7},\qquad y>0,\\
\bar\partial\Phi_{17}(\widetilde h)(x+iy) &=\frac{(iy)^{17}}{17!}\widetilde h^{(18)}(x),\qquad 0<y\le1.
\end{align*}
The extension has support in a fixed bounded set, so the contribution away from the real axis is uniformly bounded. Their representation (53) therefore implies
\begin{align*}
|Z_n^2(\widetilde h)| &\le\frac1\pi\int_0^\infty\int_{\mathbb R} |\bar\partial\Phi_{17}(\widetilde h)(x+iy)| |\mathcal B_n(x+iy)|\,dx\,dy\\
&\le C_{\widetilde h}\left(1+\int_0^1y^{10}\,dy\right)=O(1),
\end{align*}
with constants independent of $p$. Thus the second term is $O(1)$. The full bias expansion in Theorem~3 is stronger than needed here; we use only its consequence that the deterministic centering error is uniformly bounded. The spectral-edge bound allows $\widetilde h$ to be replaced by $h$ with probability tending to one. Consequently,
\begin{align*}
\sum_{i=1}^p h(\lambda_i) -p\int h(x)\,d\mu_{\mathrm{MP},\alpha_p}(x)=O_P(1).
\end{align*}
Applying this conclusion to $f_\delta$ and $g_\delta$ and dividing by $p$ verifies the two spectral conditions in Assumption~\ref{assump:subgaussian_design}. The resulting resolvent estimate is in fact stronger than required.
\end{proof}

\begin{proof}[Proof of Lemma~\ref{lem_sectionA_4}]
We first prove the ridge-norm estimate \eqref{eq:ridge_norm_rate}. Write $W=\sqrt\Delta z$, where $z\sim N(0,I_n)$ is independent of $(X,\beta^\star)$. Then
\begin{align*}
a_\Delta=A_\Delta^{-1}S\beta^\star +\sqrt\Delta A_\Delta^{-1}X^\top z.
\end{align*}
Therefore,
\begin{align*}
\|a_\Delta\|^2 &=(\beta^\star)^\top S A_\Delta^{-2}S\beta^\star +\Delta z^\top X A_\Delta^{-2}X^\top z +2\sqrt\Delta (\beta^\star)^\top S A_\Delta^{-2}X^\top z .
\end{align*}
We first record the quadratic-form identities used below. If $u$ is uniform on $S^{p-1}(\sqrt p)$ and $M$ is a deterministic symmetric matrix, then
\begin{align*}
\mathbb E[u^\top Mu]&=\operatorname{Tr}(M),\\
\operatorname{Var}(u^\top Mu) &=\frac{2p}{p+2}\left\{\operatorname{Tr}(M^2) -\frac{\operatorname{Tr}(M)^2}{p}\right\} \le 2\operatorname{Tr}(M^2).
\end{align*}
These formulas follow from rotational invariance and $\mathbb E[u_i u_j u_k u_\ell] =\frac{p}{p+2}(\delta_{ij}\delta_{k\ell} +\delta_{ik}\delta_{j\ell}+\delta_{i\ell}\delta_{jk})$, where $\delta_{ij}$ is the Kronecker delta. For $z\sim N(0,I_n)$ and a deterministic symmetric matrix $N$,
\begin{align*}
\mathbb E[z^\top Nz]=\operatorname{Tr}(N), \qquad \operatorname{Var}(z^\top Nz)=2\operatorname{Tr}(N^2).
\end{align*}
The eigenvalues of $S A_\Delta^{-2}S$ are $\lambda^2/(\lambda+\Delta)^2\le1$, whereas the nonzero eigenvalues of $X A_\Delta^{-2}X^\top$ are $\lambda/(\lambda+\Delta)^2\le(4\Delta)^{-1}$. Their squared Hilbert-Schmidt norms are therefore $O(p)$. Applying Chebyshev's inequality conditionally on $X$ yields
\begin{align*}
\frac1p(\beta^\star)^\top S A_\Delta^{-2}S\beta^\star =\frac1p\operatorname{Tr}(S^2A_\Delta^{-2})+O_P(p^{-1/2}),
\end{align*}
and
\begin{align*}
\frac{\Delta}{p}z^\top X A_\Delta^{-2}X^\top z=\frac{\Delta}{p}\operatorname{Tr}(SA_\Delta^{-2})+O_P(p^{-1/2}).
\end{align*}
It remains to control the cross term. Conditional on $(X,\beta^\star)$, it is centered Gaussian with variance
\begin{align*}
4\Delta (\beta^\star)^\top S A_\Delta^{-2}S A_\Delta^{-2}S\beta^\star.
\end{align*}
The eigenvalues of the matrix in this quadratic form are $\lambda^3/(\lambda+\Delta)^4$, which are uniformly bounded for fixed $\Delta>0$. Since $\|\beta^\star\|^2=p$, this variance is $O(p)$. Therefore,
\begin{align*}
\frac{2\sqrt\Delta}{p}(\beta^\star)^\top S A_\Delta^{-2}X^\top z=O_P(p^{-1/2}).
\end{align*}
Combining the three estimates yields
\begin{align*}
\frac1p\|a_\Delta\|^2 &=\frac1p\operatorname{Tr}(S^2A_\Delta^{-2})+\frac{\Delta}{p}\operatorname{Tr}(SA_\Delta^{-2})+O_P(p^{-1/2}) \nonumber\\
&=\frac1p\operatorname{Tr}\left(S(S+\Delta I_p)A_\Delta^{-2}\right)+O_P(p^{-1/2}) \nonumber\\
&=\frac1p\operatorname{Tr}(SA_\Delta^{-1})+O_P(p^{-1/2}) \nonumber\\
&=1-\frac{\Delta}{p}\operatorname{Tr}(A_\Delta^{-1})+O_P(p^{-1/2}).
\end{align*}
Substituting \eqref{eq:universal_resolvent} yields \eqref{eq:ridge_norm_rate}.

We next prove the residual estimate \eqref{eq:ridge_residual_rate}. Since $I_p-A_\Delta^{-1}S=\Delta A_\Delta^{-1}$, we have
\begin{align*}
Y-Xa_\Delta &=X(I_p-A_\Delta^{-1}S)\beta^\star +\sqrt\Delta (I_n-XA_\Delta^{-1}X^\top)z \nonumber\\
&=\Delta X A_\Delta^{-1}\beta^\star +\sqrt\Delta (I_n-XA_\Delta^{-1}X^\top)z .
\end{align*}
Thus
\begin{align*}
\|Y-Xa_\Delta\|^2 &=\Delta^2(\beta^\star)^\top A_\Delta^{-1}S A_\Delta^{-1}\beta^\star +\Delta z^\top (I_n-XA_\Delta^{-1}X^\top)^2z \nonumber\\
&\quad+2\Delta^{3/2}(\beta^\star)^\top A_\Delta^{-1}X^\top (I_n-XA_\Delta^{-1}X^\top)z .
\end{align*}
We use the same variance calculation for the residual terms. The eigenvalues of $A_\Delta^{-1}SA_\Delta^{-1}$ are $\lambda/(\lambda+\Delta)^2$, while $0\preceq I_n-XA_\Delta^{-1}X^\top\preceq I_n$. Thus the squared Hilbert-Schmidt norms of the two matrices defining the quadratic forms are $O(p)$ and $O(n)$, respectively. Conditional applications of Chebyshev's inequality show that both centered quadratic forms are $O_P(\sqrt p)$. The cross term is centered Gaussian conditional on $(X,\beta^\star)$, with variance
\begin{align*}
4\Delta^3(\beta^\star)^\top A_\Delta^{-1}X^\top (I_n-XA_\Delta^{-1}X^\top)^2 XA_\Delta^{-1}\beta^\star.
\end{align*}
In the singular-vector basis, the eigenvalues of the matrix between the two copies of $\beta^\star$ are $\Delta^2\lambda/(\lambda+\Delta)^4$. They are uniformly bounded, and $\|\beta^\star\|^2=p$, so the variance is $O(p)$. After division by $p$, all three centered terms are $O_P(p^{-1/2})$. Consequently,
\begin{align*}
\frac1p\|Y-Xa_\Delta\|^2 &=\frac{\Delta^2}{p}\operatorname{Tr}(SA_\Delta^{-2})+\frac{\Delta}{p}\operatorname{Tr}\left[(I_n-XA_\Delta^{-1}X^\top)^2\right]+O_P(p^{-1/2}).
\end{align*}
It remains to simplify the deterministic trace expression. Let $r:=\min(n,p)$, and let $s_1,\ldots,s_r$ denote the singular values of $X$, including zeros with multiplicity. The eigenvalues of $I_n-XA_\Delta^{-1}X^\top$ are
\begin{align*}
\frac{\Delta}{s_i^2+\Delta}
\end{align*}
for $1\le i\le r$, together with $n-r$ additional eigenvalues equal to $1$. Therefore,
\begin{align*}
&\Delta^2\operatorname{Tr}(SA_\Delta^{-2})+\Delta\operatorname{Tr}\left[(I_n-XA_\Delta^{-1}X^\top)^2\right] \nonumber\\
=&\sum_{i=1}^r\left\{\frac{\Delta^2s_i^2}{(s_i^2+\Delta)^2} +\frac{\Delta^3}{(s_i^2+\Delta)^2}\right\}+\Delta(n-r) \nonumber\\
=&\sum_{i=1}^r\frac{\Delta^2}{s_i^2+\Delta}+\Delta(n-r) \nonumber\\
=&\Delta^2\operatorname{Tr}(A_\Delta^{-1})+\Delta(n-p).
\end{align*}
Consequently,
\begin{align*}
\frac1p\|Y-Xa_\Delta\|^2 &=\frac{\Delta^2}{p}\operatorname{Tr}(A_\Delta^{-1})+\Delta\left(\frac np-1\right)+O_P(p^{-1/2}) \nonumber\\
&=\Delta E_p+\Delta(\alpha_p-1)+O_P(p^{-1/2}),
\end{align*}
where we used \eqref{eq:universal_resolvent} and $n/p=\alpha_p$. All conditional variance bounds above are uniform in $X$, with constants depending only on $\Delta$ and an upper bound for $n/p$. Averaging the conditional Chebyshev bounds over $X$ therefore gives the stated unconditional $O_P(p^{-1/2})$ estimates. This proves the lemma.
\end{proof}

\section{Sharpness of the Free-Energy and TAP Fluctuation Scale}
\label{app:fluctuation_sharpness}

\begin{proof}[Proof of the sharpness assertions in
Theorem~\ref{thm_sectionA_1}] We first prove the ridge-fluctuation lower bound \eqref{eq:ridge_fluctuation_lower} below, then transfer it to \eqref{eq:main_spherical_sharpness} and \eqref{eq:main_tap_sharpness}. Using the exact formula \eqref{eq:ridge_optimum_exact}, set $B_\Delta:=(\Delta I_n+XX^\top)^{-1}$. Write $z:=W/\sqrt\Delta$ and $\mu:=X\beta^\star$. Conditional on $(X,\beta^\star)$, $z\sim N(0,I_n)$ and
\begin{align*}
Y^\top B_\Delta Y = \mu^\top B_\Delta\mu +2\sqrt{\Delta} z^\top B_\Delta\mu +\Delta z^\top B_\Delta z .
\end{align*}
The centered linear and quadratic terms in $z$ are uncorrelated, since Gaussian third moments vanish. Consequently,
\begin{align*}
\operatorname{Var}\left( Y^\top B_\Delta Y  \middle|  X,\beta^\star \right) = 4\Delta \mu^\top B_\Delta^2\mu + 2\Delta^2\operatorname{Tr}(B_\Delta^2).
\end{align*}
Let
\begin{align*}
\mathcal E_p^{\mathrm{op}} := \{\|X\|_{\mathrm{op}}\le C_X\},
\end{align*}
where $C_X<\infty$ is the constant in Assumption~\ref{assump:subgaussian_design}, so that $\mathbb P(\mathcal E_p^{\mathrm{op}})\to1$. On $\mathcal E_p^{\mathrm{op}}$, all eigenvalues of $B_\Delta$ are bounded below by $(\Delta+C_X^2)^{-1}$. Since $n/p\ge\alpha/2$ for all sufficiently large $p$, on $\mathcal E_p^{\mathrm{op}}$ we have
\begin{align*}
\operatorname{Tr}(B_\Delta^2) \ge \frac{n}{(\Delta+C_X^2)^2} \ge \frac{\alpha p}{2(\Delta+C_X^2)^2}.
\end{align*}
Therefore, with $c_1:=\alpha\Delta^2/[4(\Delta+C_X^2)^2]>0$,
\begin{align*}
\operatorname{Var}\left( R_p  \middle|  X,\beta^\star \right) = \frac{1}{4p^2} \operatorname{Var}\left( Y^\top B_\Delta Y  \middle|  X,\beta^\star \right) \ge \frac{c_1}{p}
\end{align*}
on $\mathcal E_p^{\mathrm{op}}$.

To establish \eqref{eq:ridge_fluctuation_lower}, we use the Carbery-Wright inequality \cite[Theorem~8]{carbery2001distributional}. Applied to Gaussian measure with $q=2d$ in their notation, it states that, if $G$ is a standard Gaussian vector and $P$ is a nonzero real polynomial of degree at most $d\ge1$, then for every $\varepsilon>0$,
\begin{align*}
\mathbb P\left(|P(G)|\le \varepsilon \left(\mathbb E P(G)^2\right)^{1/2} \right) \le C_d \varepsilon^{1/d},
\end{align*}
where $C_d$ depends only on $d$. Conditionally on $(X,\beta^\star)$, we apply this inequality with $d=2$ to the polynomial
\begin{align*}
P_{X,\beta^\star}(z) := R_p-\phi(\alpha_p,\Delta).
\end{align*}
This is a polynomial of degree at most two in $z$. On $\mathcal E_p^{\mathrm{op}}$, its positive conditional variance ensures that it is nonzero. Moreover,
\begin{align*}
\mathbb E\left[ P_{X,\beta^\star}(z)^2  \middle|  X,\beta^\star \right] \ge \operatorname{Var}\left( R_p  \middle|  X,\beta^\star \right) \ge \frac{c_1}{p}
\end{align*}
on $\mathcal E_p^{\mathrm{op}}$. Hence, on $\mathcal E_p^{\mathrm{op}}$, for every $\varepsilon>0$,
\begin{align*}
\mathbb P\left( \left. \left|R_p-\phi(\alpha_p,\Delta)\right| \le \varepsilon\sqrt{\frac{c_1}{p}}  \right|  X,\beta^\star \right) \le C_2\varepsilon^{1/2}.
\end{align*}
Choose a fixed $\varepsilon>0$ such that $C_2\varepsilon^{1/2}\le1/2$, set $\eta:=1/2$, and define
\begin{align*}
c_2:=\varepsilon\sqrt{c_1}.
\end{align*}
Then, on $\mathcal E_p^{\mathrm{op}}$,
\begin{align*}
\mathbb P\left( \left. \left|R_p-\phi(\alpha_p,\Delta)\right| \ge \frac{c_2}{\sqrt p}  \right|  X,\beta^\star \right) \ge \eta .
\end{align*}
Taking expectations gives
\begin{align}
\mathbb P\left( \left|R_p-\phi(\alpha_p,\Delta)\right| \ge \frac{c_2}{\sqrt p} \right) \ge \eta \mathbb P(\mathcal E_p^{\mathrm{op}}) = \eta-o(1). \label{eq:ridge_fluctuation_lower}
\end{align}

Finally, \eqref{eq:ridge_tap_second_order} and Lemmas~\ref{lem_sectionA_1} and~\ref{lem_sectionA_3} give
\begin{align}
F_{\mathrm{TAP}}-R_p&=O_P(p^{-1})=o_P(p^{-1/2}),\nonumber\\
F_p^S-R_p&=O_P(p^{-1})=o_P(p^{-1/2}). \label{eq376}
\end{align}
For either $H_p=F_{\mathrm{TAP}}$ or $H_p=F_p^S$, the triangle inequality, \eqref{eq:ridge_fluctuation_lower}, and \eqref{eq376} give
\begin{align*}
&\mathbb P\left(|H_p-\phi(\alpha_p,\Delta)| \ge\frac{c_2}{2\sqrt p}\right)\\
&\qquad\ge\mathbb P\left(|R_p-\phi(\alpha_p,\Delta)| \ge\frac{c_2}{\sqrt p}\right) -\mathbb P\left(|H_p-R_p|>\frac{c_2}{2\sqrt p}\right)\\
&\qquad\ge\eta-o(1).
\end{align*}
Taking $c:=c_2/2$ and $\eta':=\eta/2$, we obtain
\begin{align*}
\liminf_{p\to\infty} \mathbb P\left( \left|F_{\mathrm{TAP}}-\phi(\alpha_p,\Delta)\right| \ge \frac{c}{\sqrt p} \right)&\ge\eta',\\
\liminf_{p\to\infty} \mathbb P\left( \left|F_p^S-\phi(\alpha_p,\Delta)\right| \ge \frac{c}{\sqrt p} \right)&\ge\eta'.
\end{align*}
The constants $c$ and $\eta'$ depend only on $(\alpha,\Delta,C_X)$, as required. This proves \eqref{eq:main_spherical_sharpness} and \eqref{eq:main_tap_sharpness}.
\end{proof}

\section{Gaussian Conditioning Near the Typical Radius}
\label{app:gaussian_conditioning}

\begin{proof}[Proof of Lemma~\ref{lem_gaussian_typical_sphere}]
We first prove the result in the centered-radius case $r_p^2=\mu_p$. For the general case, we introduce an exponential tilt under which $r_p^2$ is the mean squared radius and then apply the centered-radius result.

\medskip \noindent\textit{Centered-radius case.} By an orthogonal change of coordinates, we may assume that
\begin{align*}
\Sigma_p=\operatorname{diag}(c_1,\ldots,c_p), \qquad c_i\in[c_-,c_+].
\end{align*}
Let
\begin{align*}
U:=\|Z\|_2^2-\mathbb E\|Z\|_2^2, \qquad \varphi(t):=\mathbb E[e^{itU}].
\end{align*}
Writing $Z_i=a_i+\sqrt{c_i}G_i$, where $G_1,\ldots,G_p$ are independent standard Gaussian random variables, and using
\begin{align*}
\mathbb E[e^{sG_i^2+uG_i}]=(1-2s)^{-1/2}\exp\left\{\frac{u^2}{2(1-2s)}\right\},
\end{align*}
valid for $\operatorname{Re}s<1/2$ and $u\in\mathbb C$, with the square-root branch equal to one at $s=0$, we substitute $s=itc_i$ and $u=2it a_i\sqrt{c_i}$ to obtain
\begin{align*}
\mathbb E\left[e^{it(Z_i^2-a_i^2-c_i)}\right]=e^{-it(a_i^2+c_i)}(1-2itc_i)^{-1/2}\exp\left\{\frac{it a_i^2}{1-2itc_i}\right\}.
\end{align*}
Independence of the coordinates therefore gives
\begin{align*}
\varphi(t) =\prod_{i=1}^p e^{-it(a_i^2+c_i)}(1-2itc_i)^{-1/2} \exp\left\{\frac{it a_i^2}{1-2itc_i}\right\}.
\end{align*}
Taking absolute values yields
\begin{align*}
|\varphi(t)| =\prod_{i=1}^p\left[(1+4c_i^2t^2)^{-1/4} \exp\left\{-\frac{2t^2c_i a_i^2}{1+4c_i^2t^2}\right\}\right].
\end{align*}
In particular,
\begin{align}
|\varphi(t)|\le(1+4c_-^2t^2)^{-p/4}. \label{eq_gaussian_shell_cf_bound}
\end{align}
For all sufficiently large $p$, this estimate implies that $\varphi\in L^1(\mathbb R)$ and $t\varphi(t)\in L^1(\mathbb R)$. The Fourier inversion formulas used below are therefore valid.

Let $f_U$ denote the density of $U$. We claim that
\begin{align}
f_U(0)\ge \frac{c}{\sqrt p}, \label{eq_gaussian_shell_density_lower}
\end{align}
for a constant $c>0$ depending only on $(c_-,c_+,K)$. Indeed,
\begin{align*}
\sigma_p^2:=\operatorname{Var}(U) =2\operatorname{Tr}(\Sigma_p^2)+4a^\top \Sigma_pa
\end{align*}
satisfies $cp\le\sigma_p^2\le Cp$. Choose $t_0>0$ sufficiently small. On $[-t_0,t_0]$, choose the continuous branch of $\log\varphi(t)$ satisfying $\log\varphi(0)=0$. This branch is well defined because every factor in the product representation of $\varphi(t)$ is nonzero on this interval. A Taylor expansion at the origin gives, uniformly for $|t|\le t_0$,
\begin{align*}
\log\varphi(t)=-\frac{\sigma_p^2t^2}{2}+\mathcal E_p(t), \qquad |\mathcal E_p(t)|\le Cp|t|^3,
\end{align*}
where the remainder bound follows from
\begin{align*}
\sum_{i=1}^p(c_i^3+a_i^2c_i^2)\le Cp.
\end{align*}
For $t=s/\sqrt p$ and $|s|\le M_0$, the preceding expansion and the bounds $cp\le\sigma_p^2\le Cp$ imply, for all sufficiently large $p$,
\begin{align*}
\operatorname{Re}\varphi(s/\sqrt p)\ge c e^{-Cs^2}.
\end{align*}
Consequently, for each fixed $M_0\ge1$, the substitution $t=s/\sqrt p$ gives
\begin{align*}
\int_{|t|\le M_0/\sqrt p}\operatorname{Re}\varphi(t)\,dt \ge \frac{c}{\sqrt p}\int_{-M_0}^{M_0}e^{-Cs^2}\,ds \ge\frac{c_0}{\sqrt p},
\end{align*}
for all sufficiently large $p$, where $c_0:=c\int_{-1}^1e^{-Cs^2}\,ds>0$ does not depend on $M_0$. We next bound the complementary integral. Set $t_1:=(2c_-)^{-1}$. For each fixed $M_0\ge1$ and all sufficiently large $p$, the inequality $\log(1+x)\ge x/2$ for $0\le x\le1$ gives
\begin{align*}
&\int_{M_0/\sqrt p<|t|\le t_1} (1+4c_-^2t^2)^{-p/4}\,dt\\
&\qquad\le 2\int_{M_0/\sqrt p}^{\infty} \exp\left\{-\frac{c_-^2pt^2}{2}\right\}dt \le \frac{Ce^{-cM_0^2}}{\sqrt p}.
\end{align*}
For $p>4$, the change of variables $u=2c_-|t|$, together with $1+u^2\ge2u$ for $u\ge1$, yields
\begin{align*}
\int_{|t|>t_1}(1+4c_-^2t^2)^{-p/4}\,dt &\le \frac{2^{-p/4}}{c_-}\int_1^\infty u^{-p/4}\,du\\
&=\frac{2^{-p/4}}{c_-(p/4-1)}.
\end{align*}
Together with \eqref{eq_gaussian_shell_cf_bound}, these two estimates imply
\begin{align*}
\int_{|t|>M_0/\sqrt p}|\varphi(t)|\,dt \le \frac{Ce^{-cM_0^2}}{\sqrt p}.
\end{align*}
Fourier inversion gives
\begin{align*}
f_U(0)=\frac{1}{2\pi}\int_{\mathbb R}\varphi(t)\,dt =\frac{1}{2\pi}\int_{\mathbb R}\operatorname{Re}\varphi(t)\,dt.
\end{align*}
Choose $M_0$ so large that the constant in the complementary bound satisfies $Ce^{-cM_0^2}<c_0/2$. The preceding estimates then prove \eqref{eq_gaussian_shell_density_lower}. The same characteristic-function bound and Fourier inversion at an arbitrary $u\in\mathbb R$ give the uniform upper estimate
\begin{align}
\sup_{u\in\mathbb R}f_U(u) \le\frac{1}{2\pi}\int_{\mathbb R}|\varphi(t)|\,dt \le\frac{C}{\sqrt p}. \label{eq_gaussian_shell_density_upper}
\end{align}
Since $\rho_Z(t)=f_U(t-\mu_p)$ for $t>0$, this also proves \eqref{eq:global_gaussian_radius_density_upper}.

For each coordinate, set
\begin{align*}
A_i(t):=\mathbb E\left[(Z_i-a_i)e^{itU}\right].
\end{align*}
Since $Z_i\sim N(a_i,c_i)$ and $\partial_{z_i}U=2z_i$, Gaussian integration by parts gives
\begin{align*}
A_i(t)&=c_i\mathbb E\left[\partial_{z_i}e^{itU}\right]=2itc_i\mathbb E\left[Z_i e^{itU}\right]\\
&=2itc_i\left(a_i\varphi(t)+A_i(t)\right).
\end{align*}
Solving for $A_i(t)$ yields
\begin{align*}
\mathbb E\left[(Z_i-a_i)e^{itU}\right] =\varphi(t)\frac{2itc_i a_i}{1-2itc_i}.
\end{align*}
Let $\nu_i$ be the finite signed measure
\begin{align*}
\nu_i(A):=\mathbb E\left[(Z_i-a_i)\mathbf 1_{\{U\in A\}}\right].
\end{align*}
The preceding formula and \eqref{eq_gaussian_shell_cf_bound} imply $A_i\in L^1(\mathbb R)$. Hence, by Fourier inversion, $\nu_i$ has the continuous density
\begin{align*}
g_i(u):=\frac{1}{2\pi}\int_{\mathbb R}e^{-itu}A_i(t)\,dt.
\end{align*}
For $u$ in a neighborhood of zero, the coarea formula applied to $z\mapsto\|z\|_2^2-\mu_p$ gives
\begin{align*}
f_U(u) &=\int_{\|z\|_2^2=\mu_p+u} \frac{f_Z(z)}{2\|z\|_2}\,d\mathcal H^{p-1}(z),\\
g_i(u) &=\int_{\|z\|_2^2=\mu_p+u} \frac{(z_i-a_i)f_Z(z)}{2\|z\|_2}\,d\mathcal H^{p-1}(z).
\end{align*}
Initially these identities hold for almost every $u$. Parametrizing each sphere by $z=\sqrt{\mu_p+u}\,\theta$ shows that both right-hand sides are continuous near zero, so the identities hold there pointwise. Since $\|z\|_2$ is constant on each level set, the ratio of the two expressions equals the conditional expectation defined in the statement. Thus,
\begin{align*}
g_i(u)=f_U(u)\left(\mathbb E[Z_i\mid U=u]-a_i\right),
\end{align*}
whenever $f_U(u)>0$. In particular, $f_U(0)>0$ by \eqref{eq_gaussian_shell_density_lower}, and evaluation at $u=0$ yields
\begin{align*}
f_U(0)\left(\mathbb E[Z_i\mid U=0]-a_i\right) &=g_i(0)\\
&=\frac{1}{2\pi}\int_{\mathbb R} \varphi(t)\frac{2itc_i a_i}{1-2itc_i}\,dt.
\end{align*}
Since $|1-2itc_i|\ge1$, \eqref{eq_gaussian_shell_cf_bound} implies
\begin{align*}
\left|f_U(0)\left(\mathbb E[Z_i\mid U=0]-a_i\right)\right| \le C|a_i|\int_{\mathbb R}|t|(1+4c_-^2t^2)^{-p/4}\,dt.
\end{align*}
For $p>4$, direct integration gives
\begin{align*}
\int_{\mathbb R}|t|(1+4c_-^2t^2)^{-p/4}\,dt =\frac{1}{c_-^2(p-4)}.
\end{align*}
Combining this identity with \eqref{eq_gaussian_shell_density_lower} gives
\begin{align*}
\left|\mathbb E[Z_i\mid U=0]-a_i\right| \le \frac{C|a_i|}{\sqrt p}.
\end{align*}
Squaring the preceding bound and summing over $i$ yields
\begin{align}
\left\|\mathbb E[Z\mid \|Z\|_2^2=\mu_p]-a\right\|_2 \le \frac{C\|a\|_2}{\sqrt p}\le CK. \label{eq_centered_radius_mean}
\end{align}

\medskip \noindent\textit{General case.} Assume that $|r_p^2-\mu_p|\le\varepsilon_0p$. To reduce this case to the centered one, define
\begin{align*}
\psi(\lambda):=\log\mathbb E[e^{\lambda\|Z\|_2^2}].
\end{align*}
For $\lambda<1/(2c_+)$, define the tilted probability measure $\mathbb P_\lambda$ by
\begin{align*}
\frac{d\mathbb P_\lambda}{d\mathbb P}(z):=\exp\left\{\lambda\|z\|_2^2-\psi(\lambda)\right\}.
\end{align*}
Multiplying the density of $N(a,\Sigma_p)$ by $e^{\lambda\|z\|_2^2}$ and collecting the quadratic and linear terms gives a density proportional to
\begin{align*}
\exp\left\{-\frac12z^\top(\Sigma_p^{-1}-2\lambda I_p)z+z^\top \Sigma_p^{-1}a\right\}.
\end{align*}
Completing the square shows that $\mathbb P_\lambda$ is the Gaussian law $N(a_\lambda,C_\lambda)$, where
\begin{align*}
C_\lambda:=(\Sigma_p^{-1}-2\lambda I_p)^{-1}, \qquad a_\lambda:=C_\lambda \Sigma_p^{-1}a=(I_p-2\lambda \Sigma_p)^{-1}a.
\end{align*}
Moreover,
\begin{align*}
\psi'(\lambda)=\mathbb E_\lambda\|Z\|_2^2, \qquad \psi''(\lambda)=\operatorname{Var}_\lambda(\|Z\|_2^2).
\end{align*}
For $|\lambda|\le1/(4c_+)$, the eigenvalues of $C_\lambda$ lie in $[2c_-/3,2c_+]$, $\|a_\lambda\|_2\le2K\sqrt p$, and, uniformly on this interval,
\begin{align*}
\psi''(\lambda) =2\operatorname{Tr}(C_\lambda^2)+4a_\lambda^\top C_\lambda a_\lambda \asymp p.
\end{align*}
In particular, $\psi'$ is strictly increasing on this neighborhood. After decreasing $\varepsilon_0$ if necessary, its image contains $[\mu_p-\varepsilon_0p,\mu_p+\varepsilon_0p]$. Hence there is a unique $\lambda_p$ in this neighborhood such that $\psi'(\lambda_p)=r_p^2$. By construction, $r_p^2$ is the mean squared radius under $\mathbb P_{\lambda_p}$, so the centered-radius estimate applies. The mean-value theorem also gives
\begin{align*}
|\lambda_p|\le C\frac{|r_p^2-\mu_p|}{p}.
\end{align*}
Using the formula for $a_\lambda$, we also obtain
\begin{align*}
\|a_{\lambda_p}-a\|_2 \le C|\lambda_p|\|a\|_2 \le C\frac{|r_p^2-\mu_p|}{\sqrt p}.
\end{align*}

The conditional law on $\{\|Z\|_2^2=r_p^2\}$ is unchanged by the tilt. Indeed, the Radon-Nikodym derivative equals the constant $\exp\{\lambda_p r_p^2-\psi(\lambda_p)\}$ on this level set, and this factor cancels upon normalization. Moreover, throughout the chosen neighborhood of zero, the covariance matrices $C_\lambda$ have eigenvalues bounded above and below by positive constants depending only on $(c_-,c_+)$, and $\|a_\lambda\|_2\le 2K\sqrt p$. Thus the constant in \eqref{eq_centered_radius_mean} may be chosen uniformly for these tilted laws. Applying the centered-radius estimate under $\mathbb P_{\lambda_p}$ gives
\begin{align*}
\left\|\mathbb E[Z\mid \|Z\|_2^2=r_p^2]-a\right\|_2 &\le\left\|\mathbb E_{\lambda_p}[Z\mid \|Z\|_2^2=r_p^2] -a_{\lambda_p}\right\|_2+\|a_{\lambda_p}-a\|_2\\
&\le C+C\frac{|r_p^2-\mu_p|}{\sqrt p}.
\end{align*}

It remains to prove the local two-sided density bound. If $|r_p^2-\mu_p|\le M\sqrt p$, then the preceding mean-value estimate gives $|\lambda_p|\le C_Mp^{-1/2}$. Let $\rho_{\lambda_p}$ denote the density of $\|Z\|_2^2$ under $\mathbb P_{\lambda_p}$. Since $\mathbb E_{\lambda_p}\|Z\|_2^2=r_p^2$, the centered bounds \eqref{eq_gaussian_shell_density_lower} and \eqref{eq_gaussian_shell_density_upper}, applied uniformly to the tilted Gaussian laws, give
\begin{align*}
\frac{c_M}{\sqrt p}\le \rho_{\lambda_p}(r_p^2) \le\frac{C_M}{\sqrt p}.
\end{align*}
The original and tilted densities satisfy
\begin{align*}
\rho_Z(r_p^2) =\exp\{\psi(\lambda_p)-\lambda_pr_p^2\} \rho_{\lambda_p}(r_p^2).
\end{align*}
Furthermore,
\begin{align*}
0\le \lambda_pr_p^2-\psi(\lambda_p) =\int_0^{\lambda_p}u\psi''(u)\,du \le Cp\lambda_p^2\le C_M.
\end{align*}
Thus the exponential factor is bounded above and below by positive constants, which proves \eqref{eq:typical_gaussian_radius_density} and the lemma.
\end{proof}

\bibliography{ref}

\end{document}